\documentclass[sigconf]{aamas}

\usepackage{balance} %

\makeatletter
\gdef\@copyrightpermission{
  \begin{minipage}{0.2\columnwidth}
   \href{https://creativecommons.org/licenses/by/4.0/}{\includegraphics[width=0.90\textwidth]{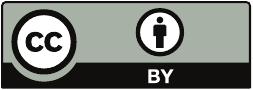}}
  \end{minipage}\hfill
  \begin{minipage}{0.8\columnwidth}
   \href{https://creativecommons.org/licenses/by/4.0/}{This work is licensed under a Creative Commons Attribution International 4.0 License.}
  \end{minipage}
  \vspace{5pt}
}
\makeatother

\setcopyright{ifaamas}
\acmConference[AAMAS '26]{Proc.\@ of the 25th International Conference
on Autonomous Agents and Multiagent Systems (AAMAS 2026)}{May 25 -- 29, 2026}
{Paphos, Cyprus}{C.~Amato, L.~Dennis, V.~Mascardi, J.~Thangarajah (eds.)}
\copyrightyear{2026}
\acmYear{2026}
\acmDOI{}
\acmPrice{}
\acmISBN{}

\doi{XUFR9329}

\acmSubmissionID{1770}

\title{Perception-Based Beliefs for POMDPs with Visual Observations}

\author{Miriam Schäfers\,\orcidID{0009-0004-9656-2389}}
\affiliation{
  \institution{Ruhr-University Bochum}
  \city{Bochum}
  \country{Germany}}
\email{m.schaefers@rub.de}

\author{Merlijn Krale\,
\orcidID{0009-0006-7194-2056}}
\affiliation{
  \institution{Radboud University}
  \city{Nijmegen}
  \country{The Netherlands}}
\email{merlijn.krale@ru.nl}

\author{Thiago D. Sim\~{a}o\,
\orcidID{0000-0001-5875-5543}}
\affiliation{
  \institution{Eindhoven University of Technology}
  \city{Eindhoven}
  \country{The Netherlands}}
\email{t.simao@tue.nl}

\author{Nils Jansen\,
\orcidID{0000-0003-1318-8973}}
\affiliation{
  \institution{Ruhr University Bochum}
  \city{Bochum}
  \country{Germany}}
\affiliation{
  \institution{Radboud University}
  \city{Nijmegen}
  \country{The Netherlands}}
\email{n.jansen@rub.de}
\author{Maximilian Weininger\,\orcidID{0000-0002-0163-2152}}
\affiliation{
  \institution{Ruhr University Bochum}
  \city{Bochum}
  \country{Germany}}
\email{maximilian.weininger@rub.de}

\begin{abstract}
Partially observable Markov decision processes (POMDPs) are a principled planning model for sequential decision-making under uncertainty. 
Yet, real-world problems with high-dimensional observations—such as camera images---remain intractable for traditional belief- and filtering-based solvers.
To tackle this problem, 
we introduce the \textbf{P}erception-based \textbf{B}eliefs for \textbf{P}OMDPs framework~($\framework$), which complements such solvers with a perception model.
This model takes the form of an image classifier which maps visual observations to probability distributions over states.
$\framework$ incorporates these distributions directly into belief updates, so the underlying solver does not need to reason explicitly over high-dimensional observation spaces.
We show that the belief update of $\framework$ coincides with the standard belief update if the image classifier is exact.
Moreover, to handle classifier imprecision, we incorporate uncertainty quantification and introduce two methods to adjust the belief update accordingly.
We implement $\framework$ using two traditional POMDP solvers and empirically show that (1) it outperforms existing end-to-end deep RL methods and (2) uncertainty quantification improves robustness of $\framework$ against visual corruption.
\end{abstract}


\usepackage[utf8]{inputenc} %
\usepackage[T1]{fontenc}    %
\usepackage{url}            %
\usepackage{booktabs}       %
\usepackage{nicefrac}       %
\usepackage{microtype}      %
\usepackage{xcolor}         %
\usepackage{float}
\usepackage[utf8]{inputenc}
\usepackage[inline]{enumitem}
\setlist[enumerate]{label=(\arabic*)}
\usepackage{centernot}
\usepackage{subcaption}

\usepackage{graphicx}
\usepackage{verbatim}
\usepackage{booktabs}
\usepackage[capitalise,noabbrev,nameinlink]{cleveref}
\usepackage{nicefrac}

\usepackage{csquotes} %

\usepackage{caption}

\usepackage{xspace}
\usepackage[textwidth=17mm,
	backgroundcolor=white,
    tickmarkheight=2pt,
	textsize=tiny]{todonotes}
\usepackage{silence}
\usepackage{cleveref}

\usepackage{listings}

\usepackage{mathtools}
\usepackage{amsthm}
\usepackage{thmtools}

\newtheorem*{lemma*}{Lemma}
\newtheorem*{problem*}{Problem Statement}

\crefname{observation}{Observation}{Observations}

\newtheorem*{corollary*}{Corollary}
\crefname{corollary}{Corollary}{Corollaries}
\theoremstyle{definition}
\newtheorem{definition}{Definition}
\newtheorem{example}{Example}
\newtheorem{assumption}{Assumption}
\crefname{assumption}{Assumption}{Assumptions}
\theoremstyle{remark}

\crefname{finding}{Result}{Results}

\usepackage[most]{tcolorbox}
\newtcolorbox{leftvrule}[1][]{colback=white,  boxrule=0pt, boxsep=0pt, breakable, enhanced jigsaw, borderline west={1.5pt}{0pt}{black},
before skip=5pt,after skip=5pt,
#1}

\usepackage{fontawesome5}
\usepackage{tikz}
  \usetikzlibrary{positioning,calc,arrows,automata,hobby,trees,shadows,fit}
  \usetikzlibrary{bayesnet}

\definecolor{_black}{cmyk}{0,0,0,0.8}
\definecolor{dblue}{HTML}{1F78B4}
\definecolor{dgreen}{HTML}{33A02C}
\definecolor{dred}{HTML}{E31A1C}
\definecolor{dorange}{HTML}{FF7F00}
\definecolor{dpurple}{HTML}{6A3D9A}

\tikzset{
  box/.style  = {rectangle,draw, minimum width=2cm, minimum height=1cm,rounded corners=2pt,align=center},
  connection/.style = {rounded corners=8pt}
}

\usepackage{color,xcolor}
\RequirePackage[T1]{fontenc}
\usepackage[utf8]{inputenc} 
\usepackage{pgfplots}
\usepackage{amsmath}
\pgfplotsset{compat=1.18}
\usetikzlibrary{pgfplots.groupplots}
\usetikzlibrary{calc}
\graphicspath{{images/}}

\usepackage{etoolbox} %
\newtoggle{arxiv}
\newcommand{\ifarxivelse}[2]{\iftoggle{arxiv}{#1}{#2}}
\toggletrue{arxiv}
\def\orcidID#1{%
  \smash{%
    \href{https://orcid.org/#1}{%
      \raisebox{-1.25pt}{\includegraphics[height=1.6ex]{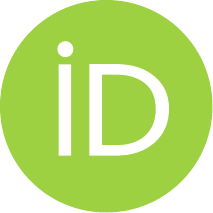}}%
    }%
  }%
}

\usepackage{ifthen}

\newboolean{showtodos}
\setboolean{showtodos}{true} %

\ifthenelse{\boolean{showtodos}}{
  \usepackage{xcolor}
  \usepackage{todonotes}
  \newcommand{\todomacro}[3]{\todo[bordercolor=#2, linecolor=#2]{\textcolor{#2}{\textbf{#1:} #3}}}
  \newcommand{\mw}[1]{\todomacro{MW}{orange!50!black}{#1}}
  \newcommand{\merlijn}[1]{\todomacro{Merlijn}{blue!75!black}{#1}}
  \newcommand{\miriam}[1]{\todomacro{Miriam}{olive!75!black}{#1}}
  \newcommand{\ts}[1]{\todomacro{TS}{purple!75!black}{#1}}
  \newcommand{\nj}[1]{\todomacro{NJ}{red!75!black}{#1}}
}{
  \newcommand{\todomacro}[3]{}
  \newcommand{\mw}[1]{}
  \newcommand{\merlijn}[1]{}
  \newcommand{\miriam}[1]{}
  \newcommand{\ts}[1]{}
  \newcommand{\nj}[1]{}
}

\newcommand{\eqdef}{\coloneqq}
\newcommand{\abs}[1]{\lvert #1 \rvert}

\DeclareMathOperator*{\argmax}{arg\,max}

\newcommand{\val}{\mathsf{V}}
\newcommand{\support}{\mathsf{supp}}

\newcommand{\vis}{\mathrm{v}}
\newcommand{\nonvis}{\neg\vis}

\newcommand{\perc}{f} %

\newcommand{\ufun}{u_f}

\newcommand{\obsvis}{\obs_{\vis}}
\newcommand{\obsnonvis}{\obs_{\nonvis}}
\newcommand{\statesvis}{\states_{\vis}}
\newcommand{\statesnonvis}{\states_{\nonvis}}
\newcommand{\statevis}{\s_{\vis}}
\newcommand{\statenonvis}{\s_{\nonvis}}
\newcommand{\obvis}{\ob_{\vis}}
\newcommand{\obnonvis}{\ob_{\nonvis}}

\newcommand{\obsfunvis}{\obsfun_{\vis}}
\newcommand{\obsfunnonvis}{\obsfun_{\nonvis}}

\newcommand{\obsfunvisapprox}{\hat{\obsfun}_\vis}
\newcommand{\modelapprox}{\mathcal{\hat{M}}}
\newcommand{\dataset}{D}
\newcommand{\Dplan}{\dataset^\text{plan}}
\newcommand{\Dperc}{\dataset^\text{perc}}
\newcommand{\Dact}{\dataset^\text{act}}

\newcommand{\states}{\mathsf{S}}

\newcommand{\obs}{\mathsf{Z}}

\newcommand{\actions}{\mathsf{A}}

\newcommand{\beliefs}{\mathsf{B}}

\newcommand{\trans}{\mathsf{T}}
\newcommand{\obsfun}{\mathsf{O}}
\newcommand{\rewardfun}{\mathsf{R}}
\newcommand{\s}[0]{\textbf{s}} %

\newcommand{\ob}{\textbf{z}} %

\newcommand{\tuq}{\texttt{TUQ}\xspace}
\newcommand{\wuq}{\texttt{WUQ}\xspace}
\newcommand{\framework}{\texttt{PBP}\xspace}

\newcommand{\hsvi}{\texttt{HSVI}}
\newcommand{\pomcp}{\texttt{POMCP}}
\newcommand{\HSVI}{\hsvi}
\newcommand{\POMCP}{\pomcp}
\newcommand{\pbphsvi}{\framework\text{-}\hsvi}
\newcommand{\thsvi}{\texttt{t}\framework\text{-}\hsvi}
\newcommand{\whsvi}{\texttt{w}\framework\text{-}\hsvi}
\newcommand{\tpomcp}{\texttt{t}\framework\text{-}\texttt{POMCP}}
\newcommand{\dqn}{\texttt{DQN}}
\newcommand{\percdqn}{\texttt{PSRL}\text{-}\dqn}
\newcommand{\perchsvi}{\texttt{PSRL}\text{-}\hsvi}
\newcommand{\oracle}{\texttt{Oracle}}
\newcommand{\noperc}{\texttt{NoPerc}}

\newcommand{\env}[1]{\emph{#1}}
\newcommand{\envtraffic}{\env{Intersection}}
\newcommand{\envlake}{\env{FrozenLake}}
\newcommand{\envflowers}{\env{FlowerGrid}}

\newlist{questionenum}{enumerate}{1}
\setlist[questionenum]{label=\textbf{Q\arabic*.}, ref=Q\arabic*, leftmargin=2\parindent, nosep}
\Crefname{question}{Question}{Questions}
\crefname{question}{question}{questions}
\crefalias{questionenumi}{question}

\newcommand{\mypar}[1]{}

\usepackage{algorithm}
\usepackage[noend]{algpseudocode}

\begin{document}

\pagestyle{fancy}
\fancyhead{}

\maketitle

\section{Introduction}

\emph{Partially observable Markov decision processes}~\citep[POMDPs;][]{KaelblingLC98} formalise sequential decision-making under uncertainty in partially observable settings~\cite[Chapter 16.4]{AIMA}.
Many applications of POMDPs, such as robotics and autonomous driving~\cite{cassandra1998survey,Kurniawati2022}, require reasoning about observation spaces based on high-dimensional images.
We call such problems \emph{vision POMDPs}~\cite[VPOMDPs;][]{DBLP:conf/l4dc/DeglurkarLTSFT23}.

\begin{figure}[tb]
  \centering
    \includegraphics[width=0.18\textwidth]{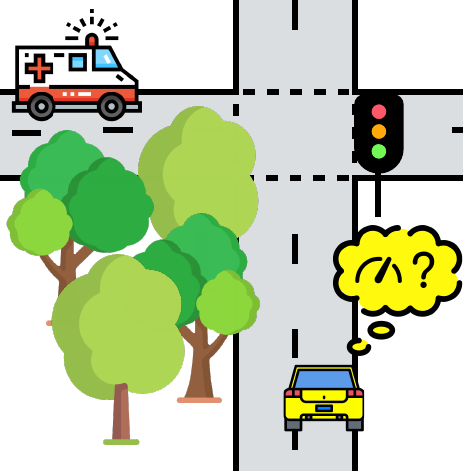}
  \caption{A vision POMDP example: 
  A car 
  must determine the presence of an ambulance using an audio sensor and the traffic light's color by interpreting camera images.
  }
  \label{fig:runningexample}
\end{figure}

\begin{example}%
    Consider %
     \label{ex: runningexample}
    an autonomous vehicle that needs to determine how to approach an intersection~(\cref{fig:runningexample}).
    Its goal is to cross only if 
    \begin{enumerate*}
        \item the traffic light is green, and
        \item no ambulance approaches the junction.
    \end{enumerate*} 
    The vehicle is equipped with a camera to detect the traffic light and an audio sensor to detect an ambulance.
    Both sensors may provide imperfect information, e.g.\ due to weather conditions, occlusion, or background noise.
    This problem can be viewed as a POMDP, where states describe the location of the car, the color of the traffic light, and the presence of an ambulance, and observations correspond to the sounds and camera images.
    Modeling the observation function would require assigning probabilities to every possible observation of the traffic light;
    in particular, the image can be subject to various sources of \emph{visual corruption}, caused by weather conditions, occlusion, etc.
\end{example}
The main challenge in solving VPOMDPs are the massive spaces of visual observations.
\emph{End-to-end deep reinforcement learning (DRL)} methods, described in the related work below, may be able to deal with such observations.
However, DRL agents are often hard to explain, do not exhibit any guarantees on their behavior, and their performance is unpredictable.
These drawbacks render them unsuitable for safety-critical settings.
In contrast, \emph{planning} methods, like SARSOP~\cite{DBLP:conf/rss/KurniawatiHL08}, POMCP~\cite{DBLP:conf/nips/SilverV10}, DESPOT~\cite{DBLP:journals/jair/YeSHL17}, and AdaOPS~\cite{DBLP:conf/nips/WuYZYLLH21}, often provide results that are easier to interpret, have guarantees on optimality, and find better policies than end-to-end learning approaches.
However, these \emph{belief-based} methods are restricted to smaller state- and action spaces and, in particular, currently cannot deal with massive observation spaces. 
We address this research gap.

\subsection*{Our Contribution}

\paragraph{Our approach.} 
This paper presents the \textbf{P}erception-based \textbf{B}eliefs for \textbf{P}OMDPs framework~($\framework$) for solving VPOMDPs.
The key idea of our approach is separating interpretation of visual observations from belief-based planning.
More precisely, $\framework$ uses a perception model to map visual observations to probability distributions over states.
We use these probability distributions to perform an efficient belief update without having to reason over the entire observation space.
This new belief update rule relies on structural assumptions that naturally arise in many VPOMDP settings, see \cref{sec:problem-statement}.
$\framework$ can integrate this belief update into any belief-based POMDP solver, allowing the resulting solver to scale to massive observation spaces.

\paragraph{Uncertainty quantification.}
Perception models are susceptible to erroneous predictions, especially in the presence of visual corruption. 
To increase the robustness of $\framework$, we use \emph{uncertainty quantification}~\citep{DBLP:conf/icml/GalG16} to estimate the uncertainty of the perception model about its prediction and introduce two methods to incorporate the uncertainty into the decision-making process:
The \textit{threshold-based} method discards predictions with high uncertainty during the belief update. 
The \textit{weighting-based} method adjusts the impact of the perception on the belief update according to its uncertainty.

\paragraph{Empirical evaluation.}
To demonstrate the feasibility of our approach, we implement two instantiations of $\framework$ using the classical POMDP algorithms $\HSVI$~\cite{DBLP:conf/uai/SmithS04} and $\POMCP$~\cite{DBLP:conf/nips/SilverV10}.
We evaluate both variants on novel benchmarks that feature small state and action spaces but complex visual observations.
Our experiments confirm that $\framework$ operates as intended and can be effectively combined with existing POMDP solvers.
In particular, $\framework$ with $\HSVI$ achieves competitive performance compared to state-of-the-art VPOMDP solvers, especially under visually corrupted image observations.

\smallskip
In summary, our main contributions are:
\begin{itemize}[nosep]
    \item 
        We introduce a novel method of performing a \textbf{belief update} in POMDPs using predictions from a perception model.
    \item 
        We use \textbf{uncertainty quantification} to make this update more robust against visual corruption.
    \item 
        We introduce the \textbf{$\framework$} framework, which allows integrating our belief update into existing POMDP solvers, and empirically show this approach is competitive with prior work.

\end{itemize}

\subsection*{Related Work}

\emph{Discretization.}
Many methods combine standard algorithms with discretizing the observation space, such as  POMCPOW~\cite{DBLP:conf/aips/SunbergK18} and DESPOT-$\alpha$~\cite{DBLP:conf/rss/GargHL19} (which both use progressive widening~\cite{DBLP:conf/lion/CouetouxHSTB11}), or the dynamic methods of~\citet{hoey2005solving}.
However, finding a discretization with small information loss is challenging for images.

\emph{Monte Carlo methods.}
Methods based on particle filters, such as AdaOPS \cite{DBLP:conf/nips/WuYZYLLH21}, LABECOP \cite{DBLP:conf/icra/HorgerK21}, SPARSE-PFT \cite{lim2023optimality}, and BetaZero \cite{DBLP:conf/rlc/Moss0CK24} use Monte-Carlo sampling to avoid reasoning over the full observation space.
However, such methods still require full knowledge of the observation function to update the particle filter, which is unrealistic for VPOMDPs.

\emph{End-to-end deep RL.}
Deep reinforcement learning methods often transform visual observations into latent representations via, e.g., convolutional layers (such as in \cite{DBLP:journals/corr/HausknechtS15}), variational auto encoders (such as in \cite{DBLP:conf/nips/HaS18}), or transformers (as discussed in \citet{DBLP:conf/nips/NiMEB23}).
Other deep reinforcement learning methods aim to learn state estimation by learning components of a particle filter through neural networks \cite{DBLP:conf/iclr/MaKHLY20}, by learning components of the belief update through variational autoencoders \cite{DBLP:conf/icml/IglZLWW18}.
These methods are trained end-to-end, which means the network needs to synchronously learn a policy and corresponding latent representation or state estimation.
This is often computationally- and data-inefficient, and makes methods hard to explain or verify.

\emph{Perception models.}
Some methods use perception models to predict state features from images.
However, it remains an open question how to best integrate such predictions into decision-making, especially in partially observable settings.
Recent work has explored ways to incorporate uncertainty quantification \cite{DBLP:journals/tits/LiuWPCYL22,DBLP:journals/tiv/LiLHHZT24}, though these methods do not incorporate history.
In contrast, our framework uses beliefs to unify perception and decision-making and takes both history and uncertainty into account.

\emph{Compositional learning.}
Visual Tree Search (VTS) \cite{DBLP:conf/l4dc/DeglurkarLTSFT23} extends DualSMC \cite{DBLP:conf/ijcai/WangL0ZD0T20} by learning generative models offline to define particle filter components, and employs a particle belief-based Monte Carlo Tree Search planner \cite{DBLP:conf/aips/SunbergK18} to select actions.
In contrast, our framework generalizes to any belief-based method (of which we consider particle filters a subset).
Moreover, our framework directly uses a perception model, which allows the take advantage of off-the-shelf \emph{Computer Vision} architectures and pretrained models to explicitly interpret the visual observations in belief construction.

\emph{Partially supervised RL.} PSRL \cite{DBLP:conf/icmla/LanierXJZV24} uses image classifier predictions as inputs to an otherwise end-to-end trained decision maker. However, these predictions are only used at the current timestep, whereas our method provides a principled way to incorporate all past predictions into a single belief.

\section{Preliminaries} %
\label{sec:prelims}
$\Delta(X)$ denotes the set of all probability distributions over a set $X$.
\subsection{Planning with State Uncertainty
}

\paragraph{POMDPs.} A  \emph{partially observable Markov decision process}~(POMDP) %
is a tuple $\mathcal{M} = \langle \states, \actions, \trans, \rewardfun, \gamma, b_0, \obs, \obsfun \rangle$, 
with finite sets of states $\states$, actions $\actions$, and observation $\obs$; 
a probabilistic transition function $\trans \colon \states \times \actions \to \Delta(\states)$; 
a reward function \( \rewardfun \colon \states \times \actions \to \mathbb{R} \); a discount factor $\gamma \in (0,1)$; an initial state distribution $b_0 \in \Delta(\states)$; and an observation function \( \obsfun \colon  \states \to \Delta(\obs) \).%
\footnote{Our method can be generalized to observation functions that also depend on actions and next states and reward functions that also depend on next states.}
We employ a factored state and observation representation with a combination of different \emph{variables}:
 $\states = \times_{i\in[1,\ldots,n]} S_i$ and $\obs = \times_{i\in[1,\ldots,m]} Z_i$, with states and observations given as tuples $\s=(s_1,\ldots,s_n) \in \states$ and $\ob=(z_1,\ldots,z_m)\in \obs$.
Single state-variables need not exactly correspond to single observation variables, hence $n\neq m$ is valid.
In our setting, some observation variables $z_i$ for $1\leq i \leq m$ are images, see Section~\ref{sec:problem-statement}.

\paragraph{Semantics of POMDPs.}
A POMDP evolves as follows:
At each time step $t\in\mathbb{N}_0$, the system is in a state $\s_t \in \states$, initially sampled according to $b_0$. 
The agent selects an action $a_t\in \actions$, obtains reward $\rewardfun(\s_t,a_t)$, and the system transitions to the next state $\s_{t+1}$ that is sampled according to $\trans(\s_t,a_t)$.
Repeating this interaction of agent and POMDP yields a \emph{path} $\rho \in (S \times A) \times (S \times A) \times \ldots$; however, in POMDPs, the agent cannot observe $\s_t$, but instead obtains observations $\ob_t \sim \obsfun(\s_t)$.
Thus, instead of~$\rho$, the agent only knows the observable \emph{history} $h \in (\obs\times \actions) \times (\obs\times\actions) \times \ldots$; 
such a history can be summarized using a \emph{belief} $b_t \in \beliefs \eqdef \Delta(\states)$, i.e. a probability distribution for the current state.

A \emph{policy} $\pi\colon \beliefs \to \Delta(\actions)$ describes the agent's behaviour by mapping beliefs to actions.
The objective of the agent is to maximize its expected cumulative discounted reward, or \emph{value}, defined as:
\begin{equation*}
    \val_\pi \eqdef \mathbb{E}_\pi \big[ \sum_{t=0}^\infty \gamma^t \rewardfun(\s_t,a_t) \big],
\end{equation*} 
where $\mathbb{E}_\pi$ is the expectation over paths under the probability measure induced by executing a POMDP with fixed policy~$\pi$.

\subsection{Perception Using Deep Neural Networks} \label{subsec: perception}
The main goal in perception is to learn a model that can reliably extract meaningful information from images.
Crucially, the model should generalize to unseen examples from the same distribution, rather than overfitting to the training data.
For this task, (convolutional) deep neural networks are the standard, outperforming traditional methods when trained on large datasets~\cite{AIMA}.

\paragraph{DNN classifiers.} 
A \emph{deep neural network (DNN) classifier}
is a function \( f
\colon X \rightarrow \Delta(L) \), mapping an input image \( x \in X \) to a probability distribution over possible classes \( l \in L\), where \( L \) is the set of possible classes~\citep{DBLP:books/daglib/0040158}.
The probability distribution is obtained by normalizing raw confidence scores through the \emph{softmax} function~\cite[Chapter 22.2.2]{AIMA}.
Typically, a single class is predicted by choosing the one with the highest probability.
DNNs are trained using supervised learning based on a dataset $D \subset X \times L$ consisting of pairs of images and classes. %
In practice, the predicted probabilities reflect the model's confidence based on its training experience, rather than true posterior likelihoods. 
However, these confidence scores often do not align with the actual probability of correctness, which we mitigate using \emph{temperature scaling}~\citep{DBLP:conf/icml/GuoPSW17}, a post-hoc calibration method that adjusts the sharpness of the softmax distribution without changing the order of the probabilities.

\paragraph{Uncertainty functions.}
{Uncertainty quantification}~\cite{DBLP:journals/air/GawlikowskiTALHFKTJRSYBZ23} evaluates the uncertainty in a DNN's predictions, allowing to ignore predictions that are likely incorrect.
Given a classifier \( \perc : X {\rightarrow} \Delta(L) \), a function \( u_{\perc} : X {\rightarrow} [0,1] \) mapping an image \( x \in X \) to a score \( u_{\perc}(x) \) is an uncertainty function of \( \perc \) if a high value of \( u_{\perc}(x) \) indicates greater uncertainty, i.e. a higher likelihood of an incorrect prediction.
Many existing uncertainty functions are surveyed in \citep{DBLP:journals/air/GawlikowskiTALHFKTJRSYBZ23}.
We use (i) the prediction confidence \(u_{\perc}(x)\eqdef 1-\max_{i =1, \ldots,|L|} \perc(l_i | x) \) and (ii) the entropy of the softmax output \( u_{\perc}(x) \eqdef - \sum_{i =1, \ldots,|L|} \perc(l_i | x) \cdot \log_2 \perc(l_i | x) \). 
Further, (iii) we employ the advanced Monte Carlo Dropout method, introduced in~\citep{DBLP:conf/icml/GalG16}. %

\section{The Vision POMDP Problem}\label{sec:problem-statement}
We define \emph{vision POMDPs} (VPOMDPs) in close alignment with the informal definition of vision POMDPs by \citet{DBLP:conf/l4dc/DeglurkarLTSFT23}. 
VPOMDPs are POMDPs in which some of the observation variables correspond to high-dimensional images.
\begin{definition}[Vision POMDP]
\label{def:perceptionPOMDP}
   Let $\mathcal{M}$ be a POMDP with  a factored state and observation representation $\states = \times_{i\in[1,\ldots,n]} S_i$ and $\obs = \times_{i\in[1,\ldots,m]}\,Z_i$.
   Let $\mathsf{I}_\obs \subseteq [1,\ldots,m]$ be the indices of the observation variables that correspond to images.
   Then $\obsvis = \times_{i\in \mathsf{I}_\obs}\,Z_i$ are the possible \emph{vision observation components} and $\obsnonvis = \times_{i\in [1,\ldots,m]\setminus \mathsf{I}_\obs }\,Z_i$ are the possible non-vision observation components.
   We call $\mathcal{M}$ a \emph{vision POMDP} if $\obsvis \neq \emptyset$, and a \emph{pure vision POMDP} if $\obsnonvis = \emptyset$.
\end{definition}
\noindent

The key challenge for VPOMDPs lies in the massive size of the observation space, which is effectively unbounded (see \cref{ex: runningexample}). 
Technically, this challenge renders it intractable to specify the probability of each possible observation, which is necessary to formally model the observation function  $\obsfun$.
We therefore assume that the observation function is not fully known, but can only be approximated using data.
To make our problem tractable, we assume we have some knowledge about which state variables can have an effect on which observation variables.

We first introduce some additional notation.
For a vision POMDP, let $\mathsf{I}_\states$ denote the set of indices that factors the state space $\states$ into vision state components $\statesvis = \times_{i\in \mathsf{I}_\states}\,S_i$ and non-vision state components $\statesnonvis = \times_{i\in [1,\ldots,n] \setminus \mathsf{I}_\states}\,S_i$, such that the vision observation components $\obsvis$ are independent of the non-vision state components $\statesnonvis$, see \cref{fig:BayesianModel} for an illustration.
We write $\statevis$ and $\statenonvis$ for vision and non-vision components of a state $\s$, and analogously $\obvis$ and $\obnonvis$ for an observation~$\ob$.
Such a factorization is always possible, since independence between vision observation components and non-vision state components trivially holds for $\mathsf{I}_\states = [1\ldots n]$.
\cref{ass:factoredObservationFunction} ensures the factorization is meaningful.

\begin{figure}[tb]
  \centering
    \resizebox{0.4\columnwidth}{!}{\def\svgwidth{120pt}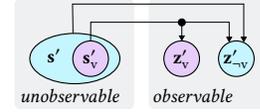}
  \caption{
  An illustration (as Bayesian network representation) of the observation function of a VPOMDP under \cref{ass:factoredObservationFunction}.
  Notably, $\obvis'$ may only depend on the visual variables of $\s'$, i.e. $\statevis'$, while $\obnonvis$ can depend on all variables of $\s'$.
  }
  \label{fig:BayesianModel}
  \Description{TODO}
\end{figure}

\begin{assumption}[Vision-factorizable.]
\label{ass:factoredObservationFunction}
A VPOMDP $\mathcal{M}$ is \emph{vision-factorizable} %
if the observation function $\obsfun$ factorizes into a vision observation function $\obsfunvis \colon \statesvis \to \Delta(\obsvis)$ and a non-vision observation function $\obsfunnonvis \colon \states \to \Delta(\obsnonvis)$, such that:
\begin{equation*}
    \obsfun(\ob \mid \s) = \obsfunvis\big(\ob_{\vis} \mid \s_{\vis} \big) \cdot \obsfunnonvis\big(\ob_{\nonvis} \mid \s\big). 
\end{equation*}

\end{assumption}

\paragraph{Interpretation}
Recall our motivating \cref{ex: runningexample}.
The vision component of the observation, $\obvis$, consists of camera images of the traffic light, while the non-vision component, $\obnonvis$, consist of audio input to detect an ambulance.
$\statevis$ describes all state components that affect $\obvis$, which is only the color of the traffic light.
The observations of the traffic light are independent of hearing the ambulance siren.
Thus, the probability of receiving a particular observation, $\obsfun(\ob \mid \s)$, can simply be written as the product of the probabilities of receiving the corresponding image and audio input, i.e. $\obsfunvis(\obvis \mid \statevis) \cdot \obsfunnonvis(\obnonvis \mid \s)$.
Thus, \cref{ass:factoredObservationFunction} holds.

\paragraph{Vision datasets}
\Cref{ass:factoredObservationFunction} allows us to decompose the observation function into two parts.
The non-vision observation function has a discrete and small domain, which means it can be learned or approximated using standard methods.
In this work, we assume the non-vision observation function is given, and instead focus on the more challenging task of dealing with the vision observation function, which cannot be fully learned due to its large, high-dimensional domain.
We assume to have access to a \emph{vision dataset}, which we define as follows:

\begin{definition}[Vision dataset]
\label{def:dataset}
A \emph{vision dataset} is a finite set of pairs 
\(
(\ob_{\vis}, \s_{\vis}) \in \obsvis \times \statesvis
\), 
where each vision component of an observation \(\ob_{\vis}\) is sampled from the conditional distribution defined by the vision observation function \(\obsfunvis\), given the vision component of the corresponding state \(\s_{\vis}\):
\[
    \dataset \subset \Big\{ (\ob_{\vis}, \s_{\vis}) \ \Big|\ \s_{\vis} \in \statesvis, \ \ob_{\vis} \sim \obsfunvis(\cdot \mid \s_{\vis}) \Big\}.
\]
\end{definition}

We highlight that \(\dataset\) is a strict subset of the full relation between vision components of states and their corresponding observations, since the complete set of all possible vision components of observations for all vision components of states is infeasibly large.  

For a given VPOMDP, a vision dataset can be constructed by sampling from the model or a simulator, if these are available.
Additionally, we can leverage existing vision datasets:
In the field of computer vision, vision datasets are made publicly available on a wide range of vision tasks, e.g. \citep{DBLP:conf/cvpr/DengDSLL009,DBLP:conf/icvgip/NilsbackZ08}. 
If public dataset images closely resemble a concrete problem (e.g., traffic light images for \cref{ex: runningexample}), we can use the images for that problem.

With such a dataset, we use the methods of \cref{subsec: perception} to train a perception model that predicts $\Pr(\statevis \mid \obvis)$, the likelihood of a state's vision component $\statevis \in \statesvis$ given an observation's vision component $\obvis \in \obsvis$.
In \Cref{sec:perc-based-belief-update}, we explain how these probabilities can be used to solve the following problem:
 
\begin{leftvrule}
\begin{problem*}[Perception POMDP problem]
Given a vision POMDP $\mathcal{M}$ satisfying \cref{ass:factoredObservationFunction}, with known model dynamics but unknown vision part of the observation function $\obsfun$, and given a vision dataset $\dataset$, compute a policy~$\pi$ that maximizes the expected cumulative discounted reward~$\val_\pi$.
\end{problem*}
\end{leftvrule}

\paragraph{Discussion of \cref{ass:factoredObservationFunction}}
Verifying Assumption 1 is part of the modeling process, as it depends on the semantic relationships between observation variables. Nevertheless, \cref{ass:factoredObservationFunction} can also be interpreted as an additional restriction on the policy rather than as an assumption about the environment:
We disallow the agent to use $\obvis$ to infer information about $\statenonvis$. %
Thus, the policies we find are valid (though possibly suboptimal) even if \cref{ass:factoredObservationFunction} is violated.
For instance, in \cref{ex: runningexample}, we may consider the location of the car to be a non-vision variable and the color of a traffic light as a vision variable.
In that case, our assumption does not hold: The images of the traffic light could be used to infer information about our location.
However, ignoring this information will still yield a valid policy that could perform reasonably well.

We emphasize that, unlike other works using factored representations~\cite{BOUTILIER200049,DBLP:conf/atal/KattOA19}, we do not require the existence of a factored representation of the transition function. 
Moreover, in contrast to approaches that aim to predict states directly from images, we do not rely on the \emph{block assumption}~\citep{DBLP:conf/icml/DuKJAD019,DBLP:conf/l4dc/SodhaniMP022,DBLP:conf/icml/ZhangSUWAS22}, which requires that each observation corresponds to just one state.
Lastly, our assumption is not related to those used for PORL (e.g., ~\cite{DBLP:journals/jmlr/SubramanianSSM22}).

\section{The Perception-based Beliefs for POMDPs Framework (\framework)}\label{sec:perc-based-belief-update}

This section formalizes the \emph{Perception-based Beliefs for POMDPs} framework~($\framework$).
\cref{subsec:beliefupdate} provides the core ingredient: the perception-based belief update that is able to handle massive observation spaces.
\Cref{subsec: uncertainty-aware} explains how we use uncertainty quantification to address an imprecise perception model.
\Cref{subsec:overall-framework} describes the overall framework, summarized in \cref{fig: model}, and embeds the perception-based belief update into various POMDP solvers.

\subsection{Belief Update for Vision POMDPs} \label{subsec:beliefupdate}
We recall the standard belief update, see e.g.~\cite[Chapter 16.4]{AIMA}. 
We denote $\Pr(\s' \mid b, a) \eqdef \sum_{\s \in \states} b(\s) T(\s' \mid \s, a)$.
For a current belief~$b_t$, the next belief $b_{t+1}$ upon selecting action $a$ and observing $\ob$ is: 
\begin{equation}
\label{eq:beliefupdatestandard}
\begin{aligned}
    b_{t+1}(\s') \eqdef %
    \frac{\obsfun(\ob \mid \s') \cdot \Pr(\s' \mid b_t, a)}{\sum_{\s'' \in \states} \obsfun(\ob \mid \s'') \cdot \Pr(\s'' \mid b_t, a)}.
\end{aligned}
\end{equation}
This update depends on the probabilities $\obsfun(\ob \mid \s)$; however, due to the massive observation space, the vision component $\obsfunvis(\obvis \mid \statevis)$ is unknown.
We rewrite $\obsfunvis$ using \emph{Bayes' rule}: %
\begin{equation}
\label{eq:obsfunbayes}
    \obsfunvis(\obvis \mid \statevis) = \frac{\Pr(\statevis \mid \obvis) \cdot \Pr(\obvis)}{\Pr(\statevis)}.
\end{equation}
$\Pr(\statevis \mid \obvis)$ can be approximated using a \emph{perception model} $\perc \colon \obsvis \to \Delta(\statesvis)$, such that $\Pr(\statevis \mid \obvis) \approx \perc(\statevis \mid \obvis)$.
We can obtain such a perception model by learning an image classifier using the vision dataset $\dataset$, defined in \cref{def:dataset}.
For the prior $\Pr(\obvis)$, the specific choice has no impact, because we prove below that the term cancels out in the belief update.
For $\Pr(\statevis)$, we choose the uninformative prior~\cite[pp. 41--43]{zellner1971}, i.e.\ a uniform distribution over $\statesvis$.
Intuitively, this means we do not introduce any bias into the approximation of the observation function, but solely rely on the perception model.
To show that this choice is appropriate, we show in \ifarxivelse{\cref{sec: belief upate 2.0}}{\cite[Appendix A]{arxivversion}} that the belief update can be viewed as a form of multiplicative \emph{opinion pooling}~\cite{10.1093/oxfordhb/9780199607617.013.37}.
In this view, we find that the uniform prior is the only choice that reproduces the standard belief update.

\paragraph{Perception-based belief update.}
We insert \cref{eq:obsfunbayes} into \cref{eq:beliefupdatestandard} to define a new belief update for vision POMDPs.
For simplicity, we first define the next belief $b_{t+1}^{pv}$ for \emph{\textbf{p}ure \textbf{v}ision POMDPs} where $\obsnonvis=\emptyset$ and thus $\s=\statevis$.
\begin{equation}
\label{eq:beliefupdatepure-derivation}
\begin{aligned}
    b_{t+1}^{pv}(\s') %
    & = \frac{ \frac{\Pr(\obvis)}{\Pr(\statevis')} \perc(\statevis' \mid \obvis) \cdot \Pr(\s' \mid b_t, a)}{\sum_{\s'' \in \states} \frac{\Pr(\obvis)}{\Pr(\statevis'')}\perc(\statevis'' \mid \obvis) \Pr(\s'' \mid b_t, a)}.
\end{aligned}
\end{equation}
As stated above, $\Pr(\obvis)$ cancels out. 
Moreover, since we assume a uniform prior over states, we have $\Pr(\statevis')=\Pr(\statevis'')$ for all pairs of states.
Thus, these also cancel out, leading to the simplified form:
\begin{equation}
\label{eq:beliefupdatepure}
\begin{aligned}
    b_{t+1}^{pv}(\s') %
    & = \frac{ \perc(\statevis' \mid \obvis) \cdot \Pr(\s' \mid b_t, a)}{\sum_{\s'' \in \states} \perc(\statevis'' \mid \obvis) \Pr(\s'' \mid b_t, a)}.
\end{aligned}
\end{equation}

In general, i.e.\ when $\obsnonvis \neq \emptyset $, we utilize the same idea and the fact that by \cref{ass:factoredObservationFunction}, we can rewrite the belief update in terms of $\obsfunvis$ and $\obsfunnonvis$ to obtain \cref{eq:beliefupdatefactoredFull}:
\begin{equation}
\label{eq:beliefupdatefactoredFull}
    b_{t{+}1}^\vis(\s_{\vis}') \eqdef 
     \frac{\perc(\s_\vis' \mid \obvis) {\cdot} \obsfunnonvis(\obnonvis \mid \s') {\cdot} \Pr(\s' \mid b_t, a)}{\smashoperator[r]{\sum_{\s'' \in \states}} \perc(\s_\vis'' {\mid} \obvis) {\cdot} \obsfunnonvis(\obnonvis {\mid} \s'') {\cdot} \Pr(\s'' {\mid} b_t, a)}.%
\end{equation}

We state in \cref{thm: factored} (proven in \ifarxivelse{\cref{app: factoredproof}}{\cite[Appendix B]{arxivversion}}) that if the image classifier was exact, \cref{eq:beliefupdatefactoredFull} allows to recover the standard belief update, thereby overcoming the problem that the massive observation space makes the standard belief update infeasible.

\begin{restatable}[Soundness of \cref{eq:beliefupdatefactoredFull}]
{theorem}{thmPerfectRocks}
    \label{thm: factored}
    Consider a vision POMDP satisfying \cref{ass:factoredObservationFunction}, and let $\perc$ be a perfect perception model, i.e.\ $\perc(\statevis {\mid} \obvis) = \Pr(\statevis {\mid} \obvis) $.
    Then, for every belief $b_t \in \Delta(\states)$, the next belief upon playing action $a\in\actions$ and observing $\ob\in\obs$ computed using \cref{eq:beliefupdatestandard,eq:beliefupdatefactoredFull} coincide, i.e.\ $b_{t+1}^\vis(\s') = b_{t+1}(\s')$.
\end{restatable}

\subsection{Uncertainty Awareness}\label{subsec: uncertainty-aware}
In practice, we do not know the true probabilities $\Pr(\statevis \mid \obvis)$, but can only learn an approximation by training a perception model $f$ on a vision dataset $\dataset$.
Errors arise from two sources: 
Firstly, the vision dataset is necessarily incomplete, and thus the perception model has to generalize to unseen images.
Secondly, ideally every image would correspond to a unique 
$\statevis \in \statesvis$.%
\footnote{This assumption is related to, but weaker than, the one used in \emph{Block MDPs}~\citep{DBLP:conf/icml/DuKJAD019,DBLP:conf/l4dc/SodhaniMP022,DBLP:conf/icml/ZhangSUWAS22}, which requires uniqueness of the entire state $\s$ rather than only its vision component~$\statevis$.
}
Realistically, however, images may be corrupted (e.g.\ the traffic light being covered in snow), such that they cannot be assigned to a unique class.
Overall, the predicted probabilities are not perfect.
Therefore, it is important to know when these probabilities are likely to be incorrect, i.e., to quantify the uncertainty of the perception model.

To address these imperfections, we employ an \emph{uncertainty-aware perception model}.
It utilizes an uncertainty function $\ufun$ to measure the uncertainty of the perception model, and, intuitively, ignores the output of the perception model in case of high uncertainty.
Ignoring the perception model means assuming that for each vision state variable $\statevis\in\statesvis$, we assign uniform probability to every value in its domain: $U(\statevis)= \frac{1}{\abs{\statesvis}}$.

We define two \emph{uncertainty-aware} perception models:
Firstly, \emph{\textbf{t}hreshold-based \textbf{u}ncertainty \textbf{q}uantification} ($\tuq$) simply uses a \emph{threshold value}~$\varepsilon$ to determine which probabilities to use, closely matching the use of uncertainty quantification in existing literature \cite{DBLP:journals/corr/abs-2410-20432}.
\begin{equation}
\label{eq:beliefupdatethreshold}
\perc_{\tuq}(\statevis \mid \obvis) = \begin{cases}
    \perc(\statevis \mid \obvis) & \text{ if } \ufun(\obvis) \leq \varepsilon \\
    U(\statevis) & \text{ otherwise.}
\end{cases}
\end{equation}
However, $\tuq$ is sensitive to the choice of $\varepsilon$.
Thus, secondly, we define \emph{\textbf{w}eighted \textbf{u}ncertainty \textbf{q}uantification} ($\wuq$) to eliminate this dependency.
$\wuq$ provides a smooth interpolation between the perception model and the uniform distribution as a function of uncertainty. When the uncertainty is below 0.5, the perception model is considered trustworthy and influences $\wuq$; if the uncertainty exceeds this threshold, $\wuq$ disregards the perception output and defaults to the uniform distribution:
\begin{equation} 
\label{eq:beliefupdateweighted}
    \perc_{\wuq}(\statevis {\mid} \obvis) {=} \begin{cases}
        \ufun(\obvis) U(\statevis) {+} \big(1{-}\ufun(\obvis)\big) \perc(\statevis {\mid} \obvis) & \text{ if } \ufun(\obvis) {<} 0.5 \\
    \Pr(\statevis) & \text{ otherwise.}
    \end{cases}
\end{equation}

Both functions degenerate to $\perc$ if $\ufun=0$; thus, \cref{thm: factored} also applies when using uncertainty-aware perception models.
In cases where the output of the perception model is ignored, the belief update relies solely on the transition probabilities and the non-vision observation function, avoiding the inclusion of potentially unreliable information from the perception model.

In practice, the perception model may yield incorrect probabilities, and its associated uncertainty may not fully reflect this, potentially leading to the following corner case:
Let the support of a distribution $d$ be defined as $\support(d) = \{x \mid d(x) > 0\}$. 
If the support of the perception model and of the belief after propagation through the transition dynamics do not overlap, the belief update produces an empty belief. In this case, we adopt a uniform distribution over all states as an uninformative \emph{fallback belief}.

\subsection{Overall Algorithm}\label{subsec:overall-framework}
\begin{figure}[tb]
  \centering
    \resizebox{0.9\columnwidth}{!}{\def\svgwidth{320pt}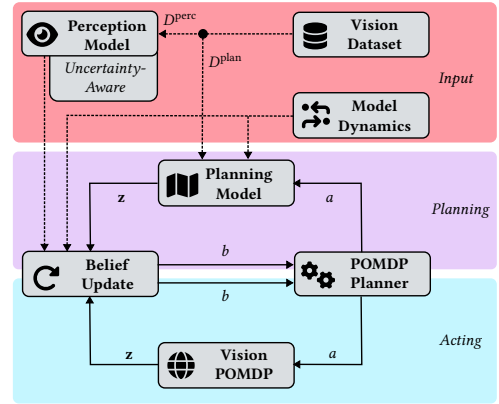}
  \caption{
  Overview of the \emph{Perception-based Beliefs for POMDPs} Framework (\framework). See \cref{subsec:overall-framework} for a detailed explanation.
  }
  \label{fig: model}
  \Description{TODO}
\end{figure}

We outline in \cref{fig: model} how the perception-based belief update is incorporated into the overall solution process of VPOMDPs.
We now provide a detailed description, split into (1) the planning cycle (middle, purple box), (2) the acting cycle (bottom, blue box), and (3) details on how different planners use the new belief update.

\paragraph{The planning cycle.}
A \textbf{POMDP Planner} maintains value estimates for different beliefs. 
Improving these estimates requires querying the next belief $b'$ upon playing action $a$ and receiving observation $\ob$ (see below and \ifarxivelse{\cref{app:planning-algos}}{\cite[Appendix C]{arxivversion}}).
However, when dealing with VPOMDPs, it is unclear how the vision component of an observation should be sampled during planning, as the distribution $\obsfunvis$ is unknown.
\framework addresses this by using an approximation of $\obsfunvis$ based on a subset $\Dplan$ of the \textbf{Vision Dataset}:
\begin{equation}
\label{eq:obsfunvisapprox}
    \obsfunvisapprox(\obvis {\mid} \statevis) = 
        \frac{|\{ (\obvis', \statevis') \in \Dplan \mid  \obvis' = \obvis, \statevis' = \statevis \}|}{|\{ ( \obvis',\statevis') \in \Dplan \mid \statevis' = \statevis \}|}. 
\end{equation}
Together with the known \textbf{Model Dynamics}, this approximation yields the \textbf{Planning Model} $\modelapprox$, a fully specified POMDP approximating the true VPOMDP.
Using this, the \textbf{POMDP Planner} can sample observations $\ob$.

In principle, the \textbf{Belief Update} could work on the \textbf{Planning Model}, using the known approximation of the observation function $\obsfunvisapprox$. 
However, as every vision component of an observation in the vision dataset is labelled with a unique $\statevis$, the planner would never encounter uncertainty about the vision component of the state. 
 Consequently, we instead apply the perception-based belief update (\cref{eq:beliefupdatefactoredFull}) already during planning, allowing the planner to reason about both the state information in the images and the uncertainty in their interpretation.
Naturally, this belief update uses the \textbf{Perception Model (with Uncertainty Quantification)}, trained%
\footnote{\framework can utilize different computer vision methods to learn the perception model. 
In particular, the training of the perception model could be bootstrapped with pretrained models if these are available for the problem setting.} 
based on the second subset $\Dperc$ of the \textbf{Vision Dataset}.

\paragraph{The acting cycle.}
The \textbf{POMDP Planner} computes a policy $\pi$.
Starting from the initial belief $b_0$, the action $a$ recommended by $\pi$ is executed in the \textbf{VPOMDP}, yielding the next observation $\ob$. 
This observation (and knowledge of the chosen action $a$ and previous belief $b$) is used by the perception-based \textbf{Belief Update} to compute the next belief $b'$.
The cycle repeats with the next action $\pi(b')$.

\paragraph{Planning with the new belief update.}
The perception-based belief update can be incorporated into many planning algorithms. 
To demonstrate this versatility, we describe how to do so for different approaches in 
\ifarxivelse{\cref{app:planning-algos}}{\cite[Appendix C]{arxivversion}}.
We provide a brief overview here.

For \emph{Point-based methods}, such as PBVI \cite{DBLP:conf/ijcai/PineauGT03}, HSVI \cite{DBLP:conf/uai/SmithS04,DBLP:conf/uai/SmithS05},  and SARSOP \cite{DBLP:conf/rss/KurniawatiHL08} the belief update affects two components of the algorithm: belief sampling and the backup operation.
We first introduce a non-standard way to represent the backup operator, which makes its use of the belief update explicit.
Then, for both belief sampling and the backup operator, we use $\modelapprox$ to approximate the possible observations and use \cref{eq:beliefupdatefactoredFull} to perform belief updates.

\emph{Particle filter based methods}, such as POMCP \cite{DBLP:conf/nips/SilverV10}, DESPOT \cite{DBLP:journals/jair/YeSHL17} and AdaOPS \cite{DBLP:conf/nips/WuYZYLLH21} use a set of particles as an approximation of the belief.
Most particle filters use the probabilities $\obsfun(\ob \mid \s)$ in the update step.
As with beliefs, we can use the perception model to approximate these probabilities.
We show that rejection sampling, such as used in POMCP, can be viewed as rejecting a particle $x$ with probability $1-\obsfun(\ob \mid x)$, which can be approximated as above. 
As in the belief update, $\Pr(\ob)$ does not affect these updates.

\emph{Deep RL} methods can utilize our belief update in a similar fashion as PSRL \cite{DBLP:conf/icmla/LanierXJZV24}. In that case, the belief represents the latent representation that is used by the agent.

We emphasize that, via \cref{thm: factored}, \framework’s replacement of the belief update preserves the solver's properties as long as the perception model is good enough.
The runtime is dominated by the chosen solver, as the (new) belief update itself is relatively inexpensive.

\section{Empirical analysis}\label{sec:experiments}

\begin{table}[tb]
\caption{
Average cumulative discounted rewards ($V$) and policy computation times ($t$) of different method on a number of benchmarks.
For the methods using $\framework$ or $\texttt{PSRL}$, training times of the perception model are not included.
}
\label{table:experiments}
\resizebox{\columnwidth}{!}{
\centering
\begin{tabular}{@{}lrrrrrrrrrrrrrrr@{}}
\toprule
&& \multicolumn{2}{c}{$\envtraffic$} && \multicolumn{2}{c}{$\envflowers$} && \multicolumn{2}{c}{$\envlake~(4)$} && \multicolumn{2}{c}{$\envlake~(8)$} \\
\cmidrule{3-4} \cmidrule{6-7} \cmidrule{9-10} \cmidrule{12-13}
\emph{Algorithms} && \multicolumn{1}{c}{$V$} & \multicolumn{1}{c}{$t$ (s)} && \multicolumn{1}{c}{$V$} & \multicolumn{1}{c}{$t$ (s)} && \multicolumn{1}{c}{$V$} & \multicolumn{1}{c}{$t$ (s)} && \multicolumn{1}{c}{$V$} & \multicolumn{1}{c}{$t$ (s)} \\ 
\cmidrule{1-1} \cmidrule{3-4} \cmidrule{6-7} \cmidrule{9-10} \cmidrule{12-13}
\texttt{PBP-HSVI}& & -3.57 & 405& & \textbf{56.6} & 412& & 0.64 & 50.5& & \textbf{0.31} & 445 \\
\texttt{tPBP-HSVI}& & -4.47 & 404& & 52.3 & 392& & 0.64 & 52.4& & \textbf{0.31} & 461 \\
\texttt{wPBP-HSVI}& & -4.11 & 406& & 56.5 & 447& & \textbf{0.65} & 53.2& & \textbf{0.31} & 454 \\
\texttt{tPBP-POMCP}& & -15.7 & 53K& & 30.8 & 105K& & 0.54 & 54K& & 0.26 & 108K \\
\cmidrule{1-1}
\texttt{PSRL-HSVI}& & \textbf{-3.41} & 303& & 45.1 & 308& & 0.64 & 13.2& & \textbf{0.31} & 318 \\
\texttt{DQN}& & -6.35 & 1.2K& & 34.5 & 3.9K& & 0.64 & 1.4K& & 0.01 & 2.1K \\
\texttt{PSRL-DQN}& & -4.03 & 1.7K& & 40.1 & 1.8K& & 0.64 & 1.7K& & \textbf{0.31} & 1.8K \\
\cmidrule{1-1}
\texttt{Oracle}& & -3.32 & 302& & 57.8 & 13.4& & 0.64 & 3.03& & 0.31 & 18.6 \\
\texttt{NoPerc}& & -15.4 & 319& & 24.3 & 325& & 0.45 & 336& & 0.23 & 334 \\
\bottomrule
\end{tabular}

}
\end{table}

In this section, we empirically evaluate our approach to answer the following questions:

\begin{questionenum}
    \item \label{q:performance} \textbf{Performance:} How does $\framework$ compare to existing baseline methods for vision POMDP tasks?
    \item \label{q:noiseresistance} \textbf{Robustness:} Is $\framework$ robust against visual corruption?
    \item \label{q:effectuq}
    \textbf{Uncertainty awareness:}
    How do different uncertainty awareness approaches affect the performance of $\framework$?
\end{questionenum}

We first give an overview of our experimental setup, then answer these questions.
Our code is available online~\cite{schafers_2026_18266744}, and further technical specifications are included in 
\ifarxivelse{\cref{app:exp-tech}}{\cite[Appendix F]{arxivversion}}. 

\subsection{Experimental Setup}

\paragraph{Benchmarks.}
We consider three environments. These have medium-sized state spaces but massive visual observation spaces that are sufficiently non-trivial to investigate the necessity and applicability of our framework.
All environments use a discount factor of $0.95$.
For further environment details, see \ifarxivelse{\cref{app: environments}}{\cite[Appendix D]{arxivversion}}.
\begin{itemize}[nosep]
    \item
        $\envtraffic$ is a custom environment that models \cref{ex: runningexample}.
        We use realistic traffic light images from~\cite{trafficlightclassifier}.
    \item 
        $\envflowers$ is a custom $5\times5$ gridworld where every state corresponds to a flower type from the 102 Category Flower Dataset \citep{DBLP:conf/icvgip/NilsbackZ08}.
        The goal of the agent is to pick a target flower and reach a goal while avoiding poisonous flowers.
        A visualization of the environment is given in \cref{fig:FlowerGrid Benchmark}.
    \item
        $\envlake$ is a variant of Gymnasium's \emph{FrozenLake} \citep{gymnasium2023}, but with an added latent variable that determines surface slipperiness.
        The agent observes this slipperiness and can infer its location using a top-down image of the environment.
        We consider the default $4\times4$ and $8\times8$ maps.
\end{itemize}

\paragraph{Vision datasets.}
For each environment, the corresponding vision dataset is split into three distinct parts: $\Dperc$ and $\Dplan$ as described in \cref{subsec:overall-framework},\footnote{For the end-to-end algorithm $\dqn$, we use $\Dperc\cup\Dplan$ for planning.}
, and $\Dact$, which is used to evaluate the computed policy.
Thus, the agent is evaluated with images that have not been used in its training, which simulates a situation where the observation space is unknown or intractably large.

\paragraph{Algorithms.}
We instantiate $\framework$ with two existing POMDP solvers: $\HSVI$ \cite{DBLP:conf/uai/SmithS05} and $\POMCP$ \cite{DBLP:conf/nips/SilverV10}, see \ifarxivelse{\cref{app:planning-algos}}{\cite[Appendix C]{arxivversion}} for implementation details.
$\thsvi$ and $\tpomcp$ utilize threshold-based uncertainty quantification (\cref{eq:beliefupdatethreshold}), with $
\epsilon=0.1$ unless noted otherwise.
$\whsvi$ uses weighted uncertainty quantification (\cref{eq:beliefupdateweighted}). 
Unless noted otherwise, we use Monte Carlo Dropout (MCDO) as uncertainty function.
\ifarxivelse{\cref{app: perception}}{\cite[Appendix E]{arxivversion}} details the DNN architectures of the perception model for each environment.
All DNNs achieve an accuracy of over $0.8$ on their test sets (see \ifarxivelse{\Cref{fig:overview_accuracy}}{Figure 11} in \ifarxivelse{\cref{app: perception}}{\cite[Appendix E]{arxivversion}}) and were trained in under $10$ minutes.
We compare to the following baselines:
\begin{itemize}[nosep]
    \item 
        $\dqn$ \cite{DBLP:journals/corr/HausknechtS15}, an end-to-end deep RL method.
        We use the \emph{Stable-Baselines3} \cite{stable-baselines3} and add memory via framestacking.
    \item 
        $\percdqn$ and $\perchsvi$, two implementation of the PSRL framework \cite{DBLP:conf/icmla/LanierXJZV24}.
        Like $\framework$, PSRL utilizes a perception model to compute probability distributions over states.
        These probabilities are interpreted as latent variables for $\percdqn$ and as beliefs for $\perchsvi$.
    \item 
        $\noperc$, a naive variant of $\HSVI$ that does not incorporate the vision component of observations in the belief update.
    \item 
        $\oracle$, an implementation of $\HSVI$ that has full observability of $\statevis$, mimicking a perfect perception model.
\end{itemize}
For all $\framework$- and $\texttt{PSRL}$-based methods, we first use the perception function to precompute probabilities for all images in our dataset.
Then, $\HSVI$-algorithms are restricted to a computational time of $300$s for policy computations, while $\tpomcp$ has a time budget of $600$s per step.
$\dqn$ and $\percdqn$ may take arbitrary time, only being restricted to $300,000$ environment interactions for training.
For the specific runtimes, see \cref{table:experiments}.
Our results represent the average after 1,000 episodes for all methods except $\tpomcp$, where we take the average over 10 episodes.

\begin{figure}[tb]
  \centering
  \includegraphics[width=0.40\textwidth]{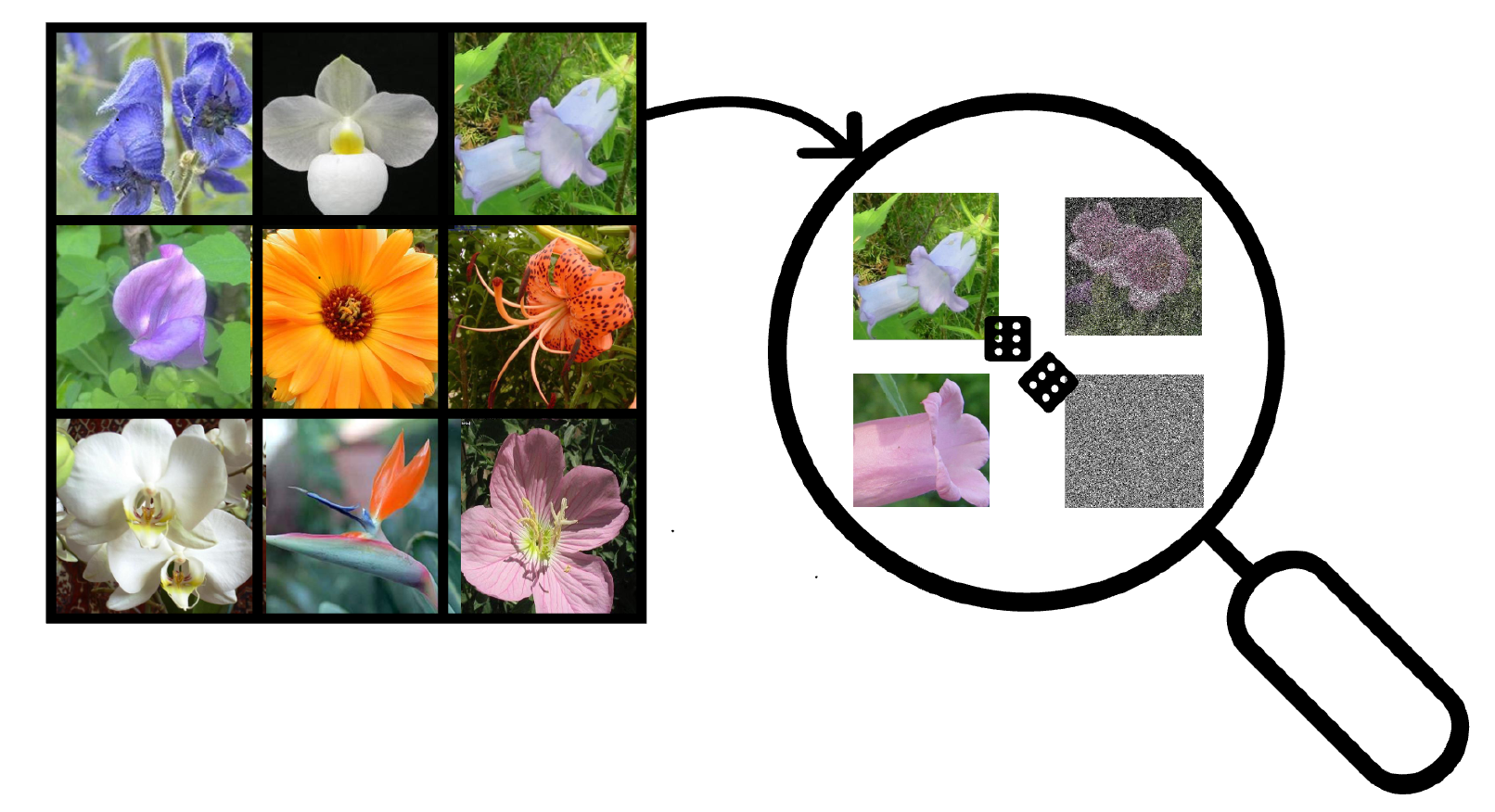}
  \caption{Visualization of $\envflowers$. 
  For each cell, the set of possible observations corresponds to images of a particular class in the 102 Category Flower Dataset \cite{DBLP:conf/icvgip/NilsbackZ08}.
  The magnifying glass shows two normal images, as well as an \emph{additive-noise} (top right) and \emph{pure-noise image} (bottom left) image.
  }
  \label{fig:FlowerGrid Benchmark}
\end{figure}

\begin{figure*}
    \centering
    \includegraphics[width=0.9\textwidth]{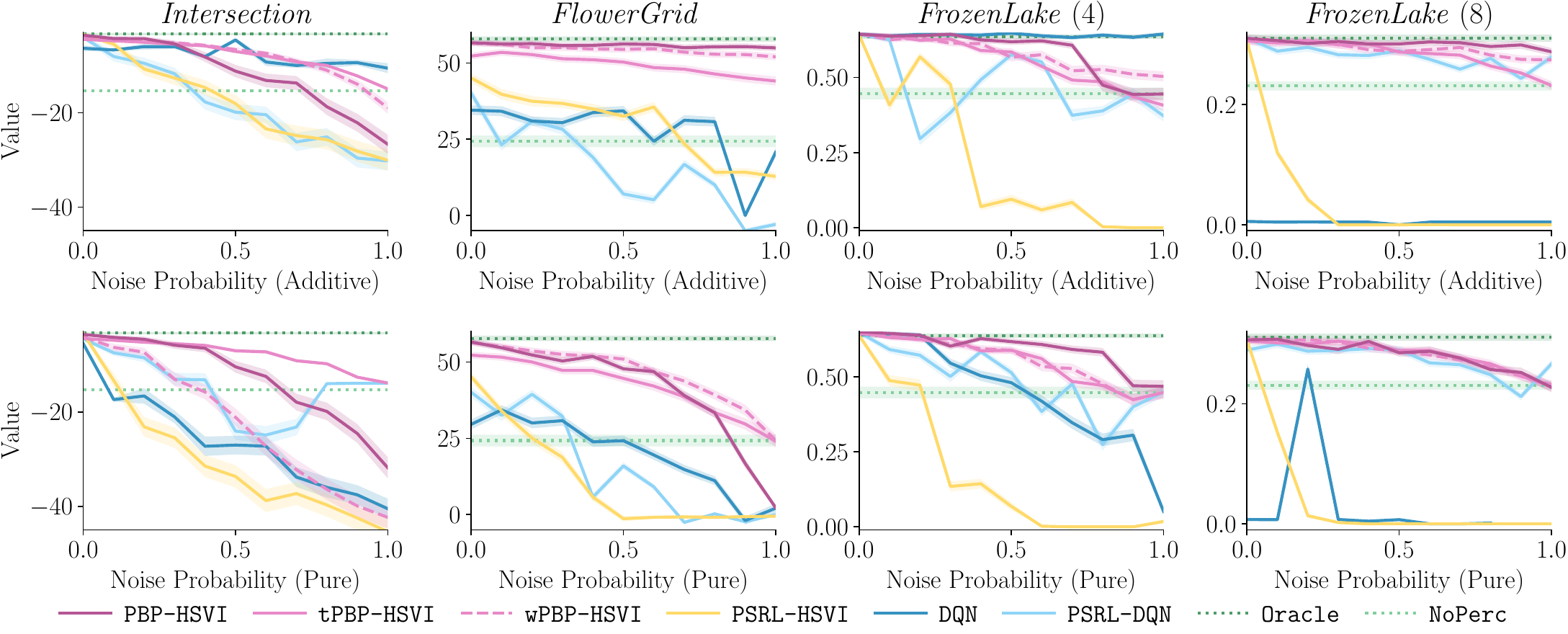}
    
    \caption{Average discounted returns (\emph{Value}) for different algorithms at different probabilities of receiving noisy observations.
    Additive noise refers to images that are correctly classified with $0.4$ probability, while full noise refers to pure salt-and-pepper images.
    Shaded areas show $95\%$ confidence in the value of the tested policy.}
    \label{fig:noise_comparison}
    \Description{TODO}
\end{figure*}

\subsection{Performance and Computational Cost}
\label{subsec:exp_performance}

To address \cref{q:performance}, we run all methods on our set of benchmarks and record both the average cumulative discounted reward~$V$ and the policy computation time $t$.
Recall that runtime depends more on the underlying solver than our framework (see end of \cref{subsec:overall-framework}).
Results are shown in \cref{table:experiments}.

{\textbf{$\framework\text{-}\hsvi$ outperforms existing methods.}
Across all benchmarks, $\pbphsvi$ performs on-par with or higher than all other algorithms, and only slightly lower than that of the $\oracle$ baseline.
Both $\thsvi$ and $\whsvi$ perform slightly worse than $\pbphsvi$, but comparably on all benchmarks except~$\envtraffic$.

The performances of the existing baselines ($\perchsvi$, $\dqn$, and $\percdqn$) is more inconsistent.
$\perchsvi$ performs roughly on-par with $\pbphsvi$ on all benchmarks except $\envflowers$.
The accuracy of our perception model is lower for $\envflowers$ than for the others, which causes more errors in the perception model.
Such errors have a larger impact on $\perchsvi$ than on $\framework$-based approaches, since these can instead use knowledge of the dynamics and previous belief to compute an accurate belief.
Despite their significantly larger time budget, $\dqn$ and $\percdqn$ never achieve higher returns than $\perchsvi$.
Moreover, $\dqn$ perform significantly worse on both $\envtraffic$ and $\envlake~(8)$.
These results suggest that combining our $\framework$ with belief-based solvers is competitive with deep learning methods for VPOMDPs.

{\textbf{$\tpomcp$ requires optimization to be competitive.}
Lastly, $\tpomcp$ outperforms $\noperc$ on all environments, but performs significantly worse than all other algorithms.
This comes from scalability limitations of the current implementation rather than any inherent weakness of the approach. 
An ablation study in \ifarxivelse{\cref{app:exp-more}}{\cite[Appendix G]{arxivversion}} supports this claim:
The particle filter of $\tpomcp$ accurately tracks the true belief across all environments, even under visual corruption (see \cref{subsec:exp_noise}).
Since $\framework$ only modifies the particle filter update, this indicates that the reduced performance stems from implementation inefficiencies rather than conceptual shortcomings.
A more optimized implementation that leverages efficient data structures
would yield substantially better results.
Therefore, we exclude $\tpomcp$ from the remaining experiments and provide additional results in \ifarxivelse{\cref{app:exp-more}}{\cite[Appendix G]{arxivversion}}.

\subsection{Robustness Against Visual Corruption}
\label{subsec:exp_noise}

Next, we investigate \cref{q:noiseresistance} by adding \emph{corrupted images} to our environments: For each dataset $\Dplan$ and $\Dact$, we create two additional variants:
\begin{enumerate*}
    \item a set of \emph{additive-noise images}, where we add salt-and-pepper noise to each image such that the perception model has an accuracy of roughly $0.4$, and 
    \item a set of \emph{pure-noise images}, where images are fully replaced by salt-and-pepper noise.
\end{enumerate*}
For an example of corrupted images, see \cref{fig:FlowerGrid Benchmark}.
Intuitively, additive-noise images are hard to classify, but the probabilities of the perception model could still contain valuable information; for pure-noise images, the probabilities from the perception model cannot be useful.
Further details on corrupted images are in \ifarxivelse{\cref{app: datasets}}{\cite[Appendix D.4]{arxivversion}}.
For our experiments, we select a portion of the images in $\Dplan$ and $\Dact$ according to a \emph{noise probability}, and replace these with corrupted variants.
The performance of our methods at different noise probabilities, for both types of corruption, is shown in \cref{fig:noise_comparison}.

\textbf{$\framework$ is robust against visual corruption.}
$\pbphsvi$, $\thsvi$ and $\whsvi$ perform well for $\envflowers$, $\envlake~(4)$, and $\envlake~(8)$ at all noise probabilities.
The performance of all $\framework$-based methods degrades more slowly than that of the other methods, and all three outperform the $\noperc$ baseline until almost all images are affected by noise.
Most notably, $\thsvi$ performs better than $\noperc$ in all our experiments.
Still, for $\envtraffic$, the performance of all three methods degrades more than in the other environments.
This is because there, the perception model is commonly overconfident for noisy images, i.e., it yields inaccurate but highly confident predictions (see \ifarxivelse{\cref{fig:overview_accuracy}}{Figure 11} in \ifarxivelse{\cref{app:acc-and-unc-perc}}{\cite[Appendix E.2]{arxivversion}}).
Such predictions cannot be detected by uncertainty quantification and lead to larger errors in the belief update.%

In comparison, the existing methods are less robust against noise:
While $\dqn$ is competitive for additive noise on $\envtraffic$ and $\envlake~(4)$, it performs worse in all other cases.
Similarly, $\percdqn$ is comparable to $\framework$ on $\envtraffic$ and $\envlake$, but is worse on $\envflowers$.
Both existing methods can perform significantly worse than just ignoring perceptions using $\noperc$.

\subsection{Uncertainty Quantification}

\begin{figure}[tb]
    \centering
    \includegraphics[width=0.9\columnwidth]{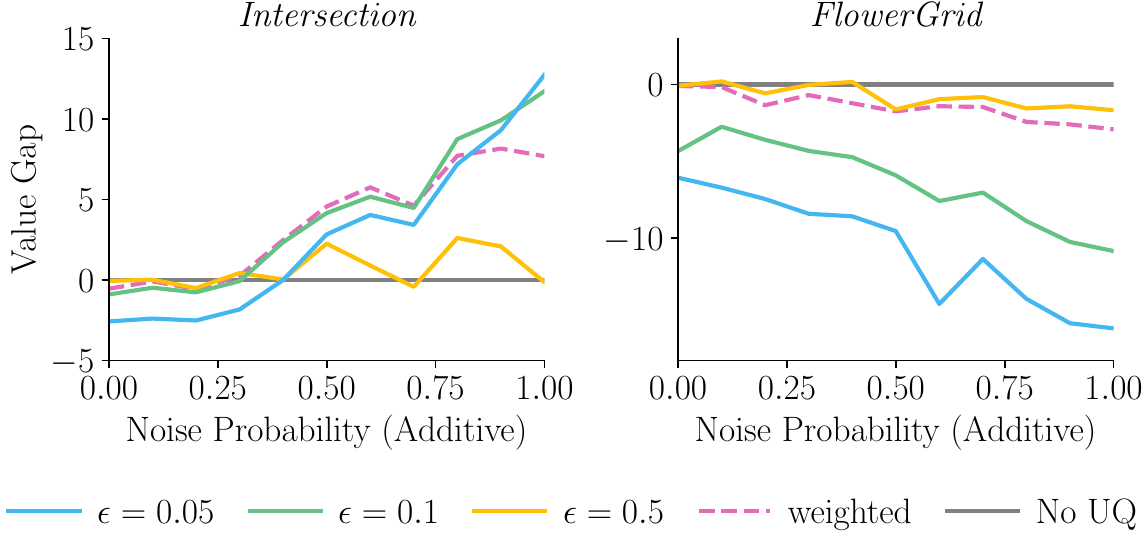}
    \caption{
    Value gap between $\pbphsvi$ and variants of $\whsvi$ and $\thsvi$ with different uncertainty thresholds, at different (additive) noise probabilities.
    }
    \label{fig:hyperparams}
    \Description{TODO}
\end{figure}

\begin{figure}[tb]
    \centering
    \includegraphics[width=0.9\columnwidth]{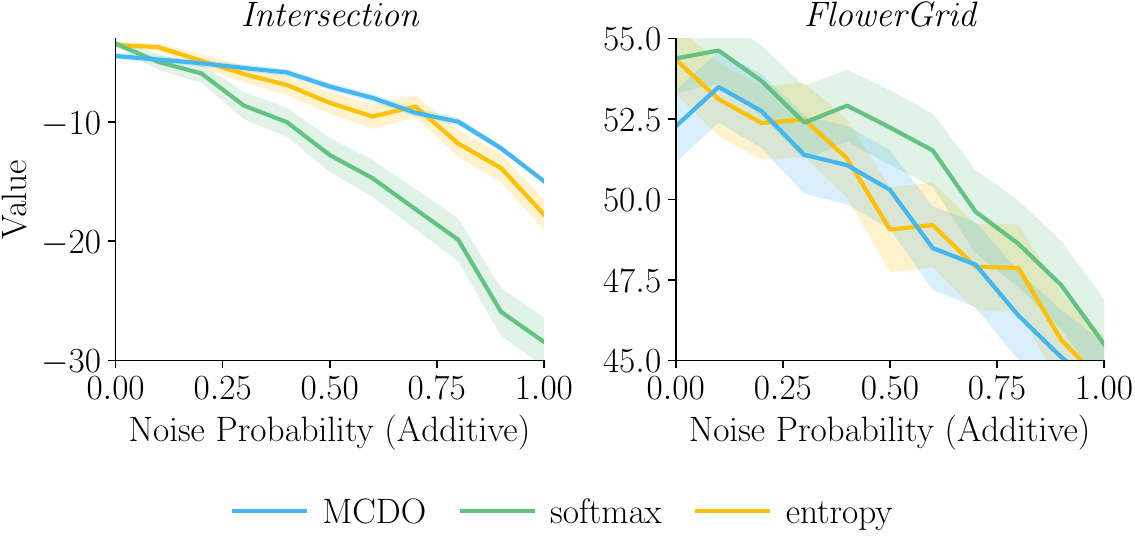}
    \caption{
    Value for $\thsvi$ using different uncertainty functions, at different (additive) noise probabilities.
    }
    \label{fig:UQmethods}
    \Description{TODO}
\end{figure}

To investigate \cref{q:effectuq}, we compare $\pbphsvi$ with different variants of our belief update and different uncertainty functions.

\textbf{Performance of UQ methods is environment-dependent.}
\Cref{fig:hyperparams} shows the value gap between the uncertainty-agnostic $\pbphsvi$ and with $\whsvi$, as well as several variants of $\thsvi$ with different thresholds $\epsilon$.
Results are mixed:
For $\envtraffic$, we find that lower thresholds improve performance at high noise probabilities, as expected.
However, for $\envflowers$, lower thresholds lead to worse performance; in particular, $\pbphsvi$ without uncertainty quantification (i.e. $\epsilon=1$) performs best.
We conjecture that the perception model for $\envflowers$ still provides a more accurate prediction than a uniform distribution, only being uncertain between a few of the many possible classes.
In such cases, even if the perception model has limited accuracy, it can be advantageous to use it.
$\whsvi$ performs well on both environments and does not require finetuning, which makes it an attractive choice.

Next, \cref{fig:UQmethods} shows the value of $\thsvi$ (with $\epsilon=0.1$) for the different uncertainty functions introduced in \cref{subsec: perception}.
All functions achieve roughly equal performance for $\envflowers$, with prediction confidence performing slightly better than the others.
However, prediction confidence performs significantly worse on $\envtraffic$.
We thus conclude that the optimal choice of uncertainty function depends on the specific environment, though MCDO and entropy seem to yield more consistent results.

\section{Conclusion}

$\framework$ is a principled and modular approach for solving vision POMDPs using belief-based POMDP solvers. 
By integrating the perception model and its associated uncertainty with the belief update, $\framework$ avoids the need to reason over massive observation spaces.
Our results demonstrate that using $\framework$ with the POMDP solver $\HSVI$ 
outperforms existing baselines.

\section*{Acknowledgements}
This work has been
partially funded by the ERC Starting Grant DEUCE (101077178).

\bibliographystyle{ACM-Reference-Format} 
\bibliography{References}

\clearpage
\ifarxivelse{
\appendix

\section{Further Justification for Uniform Prior Over States in Perception-Based Belief Update}
\label{sec: belief upate 2.0}

In the perception-based belief update (\cref{sec:perc-based-belief-update}), we are required to specify a prior over states.
This prior is the assumed probability distribution over states before having any evidence via observations. 
Using the uniform distribution is the standard uninformative prior~\cite[pp. 41--43]{zellner1971}.

Still, while the uniform distribution is the least informative, it introduces some information and thus its use has been questioned, as noted in~\cite{zellner1971} and discussed in detail in, e.g.,~\cite{eva2019principles}.
Thus, we provide a further justification without using the uniform prior by viewing the problem through the lens of \emph{opinion pooling} instead.
In particular, the belief update can be viewed as combining the information from two separate sources:
\begin{enumerate*}
    \item the previous belief plus the dynamics, and
    \item the received observation.
\end{enumerate*}
If we were to update considering only the former, then our (simplified) belief update would look as follows:
\begin{equation}
    \hat{b}_{t+1}(\s') =  \Pr(\s' \mid b_t, a)
\end{equation}
If we only consider the last observation instead, our beliefs would match exactly those of our perception model, i.e., \ be equal to $\perc(\statevis' \mid \obvis)$.
Following \cite{10.1093/oxfordhb/9780199607617.013.37}, the optimal method of combining these beliefs is to use \emph{multiplicative pooling}, which yields:
\begin{equation}
\label{eq:beliefupdatepureAlternative}
\begin{aligned}
    b_{t+1}^{pv}(\s') %
    & = \frac{ \perc(\statevis' \mid \obvis) \cdot \Pr(\s' \mid b_t, a)}{\sum_{\s'' \in \states} \perc(\statevis'' \mid \obvis) \Pr(\s'' \mid b_t, a)}.
\end{aligned}
\end{equation}

\Cref{eq:beliefupdatepureAlternative} is exactly \cref{eq:beliefupdatepure}, and thus provides an alternative way of deriving the belief update in the pure vision setting.
The general belief update can be derived analogously.

\section{Proof of Theorem \ref{thm: factored}} \label{app: factoredproof}
\thmPerfectRocks*

\begin{proof} 

First, for an arbitrary state $\s$ and observation $\ob$, we decompose the observation function and insert the perfect perception model, combining the idea of \cref{eq:obsfunbayes} and \cref{ass:factoredObservationFunction}.

\begin{align*}
    \obsfun(\ob \mid \s) &= \obsfunvis\big(\ob_{\vis} \mid \s_\vis \big) \cdot \obsfunnonvis\big(\ob_{\nonvis} \mid \s\big) 
    \tag{\cref{ass:factoredObservationFunction}}
    \\
    &= \frac{\Pr(\statevis \mid \obvis) \cdot \Pr(\obvis)}{\Pr(\statevis)} \cdot \obsfunnonvis\big(\ob_{\nonvis} \mid \s\big) 
    \tag{\cref{eq:obsfunbayes}}
    \\
    &= \frac{\perc(\statevis \mid \obvis) \cdot \Pr(\obvis)}{\Pr(\statevis)} \cdot \obsfunnonvis\big(\ob_{\nonvis} \mid \s\big) 
    \tag{$\perc$ being perfect}
\end{align*}

Inserting this into \cref{eq:beliefupdatestandard} and rewriting yields our goal, namely the equality with \cref{eq:beliefupdatefactoredFull}.

\begin{align*}
    &b_{t+1}(\s') \eqdef 
    \frac{\obsfun(\ob \mid \s') \cdot \Pr(\s' \mid b_t, a)}{\sum_{\s'' \in \states} \obsfun(\ob \mid \s'') \cdot \Pr(\s'' \mid b_t, a)}
    \tag{\cref{eq:beliefupdatestandard}}
    \\
    &= \frac
    {\frac{\perc(\statevis' \mid \obvis) \cdot \Pr(\obvis)}{\Pr(\statevis')} \cdot \obsfunnonvis(\obnonvis \mid \s') \cdot \Pr(\s' \mid b_t, a)}
    {\smashoperator[r]{\sum_{\s'' \in \states}} \frac{\perc(\statevis'' \mid \obvis) \cdot \Pr(\obvis)}{\Pr(\statevis'')} \cdot \obsfunnonvis(\obnonvis \mid \s'')  \cdot \Pr(\s'' \mid b_t, a)}
    \tag{Inserting the above}
    \\
    &= \frac{\frac{\Pr(\obvis)}{\Pr(\statevis)} 
    \cdot \perc(\statevis' \mid \obvis) \cdot \obsfunnonvis\big(\ob_{\nonvis} \mid \s'\big) \cdot \Pr(\s' \mid b_t, a)}
    {\frac{\Pr(\obvis)}{\Pr(\statevis'')} 
    \cdot \smashoperator[r]{\sum_{\s'' \in \states}} 
    \perc(\statevis'' \mid \obvis) \cdot \obsfunnonvis\big(\ob_{\nonvis} \mid \s''\big)  \cdot \Pr(\s'' \mid b_t, a)}
    \tag{Reordering and using $\Pr(\statevis) = \Pr(\statevis'')$ (uniform prior)}
    \\
    &= \frac{\perc(\statevis' \mid \obvis) \cdot \obsfunnonvis\big(\ob_{\nonvis} \mid \s'\big) \cdot \Pr(\s' \mid b_t, a)}
    {\smashoperator[r]{\sum_{\s'' \in \states}} 
    \perc(\statevis'' \mid \obvis) \cdot \obsfunnonvis\big(\ob_{\nonvis} \mid \s''\big)  \cdot \Pr(\s'' \mid b_t, a)}
    \tag{Cutting $\frac{\Pr(\obvis)}{\Pr(\statevis)}$}
    \\
    &\eqqcolon b_{t{+}1}^\vis(\s') \tag{\cref{eq:beliefupdatefactoredFull}}
\end{align*}

\end{proof}

\section{Algorithms}\label{app:planning-algos}

This section consists of two parts.
First, we give a conceptual explanation as to how $\framework$ can be incorporated into different planning algorithms.
Second, we provide technical details of the algorithms we used in our experimental evaluation.

\subsection{Planning with $\framework$}

\paragraph{Point-based value iteration.}
We consider $\hsvi$, a classic pomdp solver based on point-based value iteration.
In $\hsvi$ \cite{DBLP:conf/uai/SmithS04}, the observation function is invoked in two ways.
Firsly, during the exploration phase (Algorithm 2, step 2), $\hsvi$ samples possible next observations according to the excess uncertainty heuristic.
We use $\obsfunvisapprox$ to do this, but use \cref{eq:beliefupdatefactoredFull} to compute the corresponding beliefs.
Secondly, the backup operation (Algorithm 3) explicitly uses the observation function.
Using the same notation of \cite{DBLP:conf/uai/SmithS04}, the backup is defined as follows:
\begin{equation}
\begin{aligned}
    \beta_{a,o} &\gets \argmax_{\alpha \in \Gamma} (\alpha \cdot \tau(b,a,\ob)), \\
    \beta_a(\s)& \gets \rewardfun(\s, a) + \gamma \sum_{o,s'} \beta_{a,o}(s') \obsfun(\ob \mid \s') \trans(\s' \mid \s, a), \\
    \text{backup}(\Gamma, b)& \gets \argmax_{\beta_a}(\beta_a \cdot b),
\end{aligned}
\end{equation}
where $\tau(b,a,\ob)$ denotes belief $b'$, as reached from $b$ after action $a$ and observation $\ob$.
To incorporate this into $\framework$, we replace $\obsfun(\ob \mid \s')$ with $\obsfunvisapprox(\obvis \mid \statevis) \obsfunnonvis(\obnonvis \mid \s)$, and redefine $\tau(b,a,\ob)$ so that it matches the belief update of \cref{eq:beliefupdatefactoredFull}.

\paragraph{Particle filters.}

To illustrate how $\framework$ can be applied to particle filters, we consider two different methods.
Firts,  \emph{sequential importance sampling} (SIS) \cite{gordon1993novel} uses a weighted particle filter, which consists of a set of particles $X_t \in \{ x^i_t \mid x_t^i \in S \}$ with corresponding weights $W_t \in \{w^i_t \mid w^i_t \in \mathbb{R}$.
At each timestep, new particles get sampled and receive new weights as follows:
\begin{equation}
\begin{aligned}
\label{eq:particlefiltersample}
    x_{t+1}^i &\sim \trans(\cdot \mid s^i_{t},a)\\
    w_{t+1}^i &= C \cdot w_{t}^i \cdot \obsfun(\ob \mid x_{t+1}^i),
\end{aligned}
\end{equation}
with $C$ a normalization constant.
As in \cref{eq:beliefupdatefactoredFull}, we use \cref{ass:factoredObservationFunction} to replace $\obsfun$ with $ \frac{\Pr(\obvis)}{\Pr(\statevis')} \perc(x^i_{t+1} \mid \obvis) \obsfunnonvis(\obnonvis \mid x^i_{t+1})$.
Similar to the belief update, we see that $\Pr(\obvis)$ can be absorbed into the normalization constant $C$.
This approach can be applied to a wide range of particle filtering methods that use similar approaches \cite{DBLP:journals/tsp/ArulampalamMGC02,DBLP:journals/sac/DoucetGA00}. 

\begin{algorithm}[tb]
\caption{\textsc{POMCP Particle Filter Update}}
\label{alg:pomcpupdate}
\begin{algorithmic}[1]
\State \textbf{Input:} Particle set $X_t$, observation $\ob_t$, action $a_t$
\State $X_{t+1} = \emptyset$
\While{$|X_{t+1}| < K$}
    \State $x^i_t \sim X_t$
    \State $(x', \ob', r') \sim \mathcal{G}(x^i_t,a)$
    \If{$\ob = \ob'$}
        \State $X_{t+1} = X_{t+1} \cup \{ x' \}$
    \EndIf
\EndWhile
\State \textbf{return } $X_{t+1}$
\end{algorithmic}
\end{algorithm}
\footnotetext{
The particle filter update for $\pomcp$
}

Next, we consider the method used by the classic POMDP solver $\pomcp$.
As shown in \cref{alg:pomcpupdate}, $\pomcp$ uses a set of unweighted particles, and uses a generator $\mathcal{G}$ to approximate the model dynamics.
However, even with access to a perfect generator, the probability of sampling a matching observation is very small for vision POMDPs, making this approach intractable.
Instead, we note that with this sampling approach, the  probability of any particle $x'$ getting added to $X_t$ at each iteration can be separated as $\sum_{\s \in \states} \Pr(x^i_t = \s \mid \mathcal{G}, X_t) \Pr(x' \mid \mathcal{G}, \s, a) \Pr(\ob \mid \mathcal{G}, x')$.
We propose to use sampling for the first two probabilities, while we use approximate $\Pr(\ob \mid \mathcal{G}, x')$ using the perception model as before.
Concretely, this means replacing line 6 with the following two lines:
\begin{equation*}
\begin{aligned}
    &p \sim \text{Uniform}(0,1) \\
    &\textbf{if } p < \frac{\Pr(\obvis)}{\Pr(\statevis')} \perc(x' \mid \obvis) \cdot \obnonvis(\obnonvis \mid x') \textbf{ then }
\end{aligned}
\end{equation*}
We note that $\Pr(\obvis)$ is a factor for each possible particle $x' \in \states$.
Thus, it affects the probability of each particle being chosen equally, and will not have an impact on $X_t$.

\subsection{Implementation and Configuration}

Next, we provide a short description with technical details for all used algorithms. 
All our code is included in the supplementary materials.

\paragraph{\hsvi.}
We use a Python-based custom implementation of HSVI.
Our implementation roughly follows the description of HSVI \cite{DBLP:conf/uai/SmithS04}, with a number of minor alterations:
\begin{enumerate}[nosep]
    \item We use the sawtooth approximation \citet{DBLP:journals/jair/Hauskrecht00} for computing upper bounds, and use a sparse implementation of $\alpha$-vectors. Both of these alterations were previously introduced HSVI2 \cite{DBLP:conf/uai/SmithS05}.
    \item Due to our relatively large planning observation space, our belief tree has a large branching number.
    To deal with this, we implement the following:
    \begin{enumerate*}
        \item 
            We prune all subtrees in our belief tree that are provably suboptimal, i.e. all beliefs $b'$ reached via a belief-action pair $(b,a)$ and where we have:
            $ \exists a' \colon \overline{Q}(b,a) < \underline{Q}(b,a')$.
            This method originates from the POMDP-solver SARSOP \cite{Kurniawati2022}.
        \item 
            We prune both beliefs and $\alpha$-vectors at each iteration.
    \end{enumerate*}
    \item 
        To allow for more broad exploration of the belief space, we employ $\epsilon$-optimal exploration: we pick actions and observations using HSVI's forward exploration heuristics with probability $1-\epsilon_{explore}$, and pick randomly otherwise.
        Our experiments use $\epsilon_{explore} = 0.1$.
\end{enumerate}

\paragraph{\pomcp.}
We use a refactored version of a publically available Python implementation of $\pomcp$.
\footnote{For the original code, see: \url{https://github.com/namoshizun/PyPOMDP}.}
We use a 1000 particles with an invigoration rate of $0.05$, and use UCB1 \cite{DBLP:journals/ml/AuerCF02} to guide exploration.
For the rollout policy, we pick a random action with probality $0.2$.
Otherwise, we sample a random state $s'$ using the probabilities of the perception model and take the action that would be optimal for $s'$ in an MDP.

\paragraph{\dqn.}
Both tested variants of $\dqn$ use the stable baselines~ 3 implementation \cite{stable-baselines3}.
In particular, $\dqn$ uses the standard CNN policy type, using an image size of $64x64$ pixels.
$\percdqn$ uses the standard MLP policy type, using the prediction of a pre-trained perception model as input.
To encourage exploration, we set the exploration fraction to $0.8$, initial exploration rate to $0.8$ and final exploration rate to $0.05$.
For all other variables we use the default setting.

\section{Benchmarks} \label{app: environments}

In this section, we provide a detailed description of all used benchmarks.
Before, we give a high-level overview of their structure.

\subsection{FrozenLake} \label{app: frozenlakeextended}

This environment extends the classical FrozenLake MDP from the \texttt{gymnasium} library  \citep{gymnasium2023} by introducing partial observability, dynamic slippery conditions, and stochastic visual observations.

The agent navigates on a $4 \times 4$ or $8\times8$ grid. The true state consists of two components:
\begin{itemize}
    \item The agent’s discrete position on the grid.
    \item A binary slippery variable indicating whether the environment is slippery.
\end{itemize}

\paragraph{Observations.}

The agent does not directly observe the discrete state. 
Instead, it receives:
\begin{itemize}
    \item An image representing the agent’s current grid cell. 

    \item A deterministic binary sensor reading of the slippery variable.
\end{itemize}

\paragraph{Transition Dynamics.}

\begin{itemize}
    \item The agent can take actions \emph{north}, \emph{east}, \emph{south} and \emph{west}.
    \item When not slippery, actions deterministically move the agent in the given direction.
    \item When slippery, the agent’s instead moves to one of the adjacent cardinal directions with a $0.5$ probability
    \item The slippery state changes with probability $0.5$ at each time step to slippery, independently of the agent’s action.
\end{itemize}

\paragraph{Reward Structure.}

The agent receives a reward of 1 upon reaching the goal state, and zero reward otherwise. Episodes terminate when the agent reaches the goal or falls into a hole, analogous to the classical FrozenLake environment.

\subsection{FlowerGrid} \label{app: flowers25}

The \emph{FlowerGrid} environment is defined on a $5 \times 5$ grid, where each cell corresponds to a unique flower class drawn from the 102 Category Flower Dataset \citep{DBLP:conf/icvgip/NilsbackZ08}, specifically classes 1 through 25. Example images of the dataset are shown in \cref{fig: flowers25}.
The true state consists of two components: the agent’s current position (represented using row-major ordering) and a binary variable indicating whether the target flower has already been picked.
The agent begins in the top-left cell and aims first to navigate the grid to locate and pick exactly one instance of a specific target flower class.
Once this flower is picked, the agent must reach the goal position in the bottom-right cell to successfully complete the episode.
Toxic flowers are scattered around the target flower (at positions $2,5,8,11,17,19,20,23$)
while the target flower is positioned at cell 20.
The episode terminates only after the agent reaches the goal cell located at position $25$.

\paragraph{Observations.}
The agent does not directly observe the full state. 
Instead, it receives an image depicting the flower class at its current location, sampled from the set of possible images for this class.
Representative examples for state 20 are shown in \cref{fig: flowers25_state20}. 
Additionally, the agent has full access to a binary variable indicating whether the target flower has been collected.

\paragraph{Transitions \& reward.}
Transitions in this environment are stochastic to reflect uncertainty in movement.
For movement actions—up, down, left, and right—the agent moves to the intended neighboring cell with probability $0.6$, assuming the move is within grid boundaries.
With a complementary $0.4$ probability, the agent transitions randomly either to its current state or to one of the neighboring states, simulating uncertainty analogous to slipperiness.
The pick action modifies the state only when the agent is located at the target flower’s position and has not previously picked it. 
In this case, the target flower is marked as picked, and the agent receives a reward of +10. If the agent attempts to pick a poisonous flower, the episode terminates immediately with a large negative reward of -10.
Attempts to pick any other flower or to pick the target flower multiple times yield a penalty of -1 without changing the state.

After successfully picking the target flower, the agent must navigate to the goal cell located in the bottom-right corner. 
Upon arrival, the episode ends, awarding the agent a substantial positive reward of +100.

\subsection{Intersection} \label{app: trafficlightenv}

Based on \cref{ex: runningexample}, the \emph{Intersection} environment models a discrete MDP over a linear road segment with six non-terminal and one terminal positions, combined with a traffic light state and siren status. 
The state space is the Cartesian product of three components: the traffic light color (green, red, yellow), the agent’s position along the road (positions $0$ through $5$ plus $-1$ indicating terminal), and a binary siren status (on or off).

\paragraph{Observations.}  
The agent receives a composite observation consisting of three components:

\begin{itemize}
  \item \textbf{Visual traffic light color:} An image depicting the current traffic light color (green, red, or yellow), randomly sampled from a set of images associated with that class. These images come from a traffic light dataset \cite{trafficlightclassifier}, where each image is labeled according to its light state.
  
  \item \textbf{Deterministic position:} The agent always observes its exact position along the road. The position is a discrete value between 0 and 5, or $-1$ indicating that the agent has reached the terminal state.
  
  \item \textbf{Noisy siren observation:} the agent receives a potentially noisy indication of whether a siren is active. 
  If the true siren state is \emph{coming}, then the agent always observes “coming”, but they can observe either “coming” or “none” with probability $0.5$ if not the siren state is \emph{none}.
  This models limited perceptual accuracy in detecting siren sounds.
\end{itemize}

\paragraph{Transitions \& reward.}  
Transition uncertainty arises both in the behavior of the traffic light and the siren. 
The traffic light changes state according to predefined probabilities depending on its current color.
More precisely, if the light is red, it remains red with probability $0.8$ or changes to yellow probability $0.3$;
if yellow, it always changes to green;
and if green, it remains green with probability $0.6$ or switches to red with $0.4$ probability. 
The siren toggles stochastically between active and inactive states: if the siren is currently off, it stays off with probability $0.8$ and turns on with probability $0.2$;  it is on, it remains on with probability $0.8$ and switches off with probability $0.2$. 
The agent can choose to wait or move backward by one or two positions.
Rewards are given based on reaching the terminal position -1:  
\begin{itemize}
  \item Reaching terminal position under green or yellow light without siren: 0 reward.
  \item Reaching terminal position under red light: -100 penalty.
  \item Additional penalty of -200 if the siren is active upon arrival.
  \item Waiting (action 0) incurs a small negative reward of -1.
\end{itemize}

\subsection{Vision datasets} \label{app: datasets}

\paragraph{FrozenLake}
To construct the vision dataset for the \textit{ FrozenLake} environment, we generated images for all possible state-action pairs in the $4\times4$ and $8\times8$ grid world.
Each state-action pair was rendered as an RGB image, forming the complete observation space for the perception model. 
The resulting dataset contains 384 and 1536 images, respectively, of which we use $0.5$ for training, $0.1$ for validation, and $0.4$ for testing.

\paragraph{FlowerGrid}
The vision dataset for the \emph{FlowerGrid} environment is derived from the publicly available 102 Category Flower Dataset \citep{DBLP:conf/icvgip/NilsbackZ08} by selecting only the first 25 classes.
We retained the original dataset split.
\paragraph{Intersection} For the \emph{Intersection} environemnt, we use the dataset from \cite{trafficlightclassifier}, which contains images of traffic lights classified into three categories: red, yellow, and green. We adopt the original train-test split provided in the repository for our experiments. 
Example images are visualized in \cref{fig: trafficlights}.

\paragraph{Visual corruption.}
For all datasets, visual corruption is introduced using salt-and-pepper noise. Specifically, a subset of image pixels is randomly selected according to an environment-specific noise ratio, and each selected pixel is then set to either white or black with equal probability.
We use a noise ratio of $0.4$ for $\envtraffic$, $0.3$ for $\envflowers$,
$0.2$ for $\envlake(8)$,
and $0.3$ for $\envlake(4)$.

\begin{figure*}[tp]
  \centering
  \includegraphics[width=0.8\textwidth]{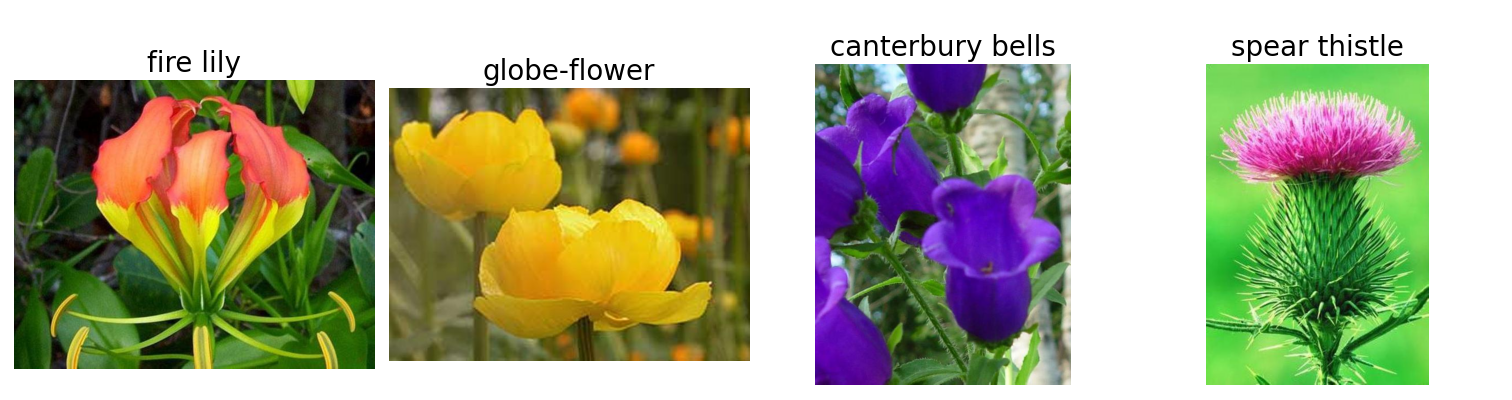}
  \caption{Four images of the 102 Category Flower Dataset \citep{DBLP:conf/icvgip/NilsbackZ08}. }
  \label{fig: flowers25}
  \Description{TODO}
\end{figure*}
\begin{figure*}[tp]
  \centering  \includegraphics[width=0.8\textwidth]{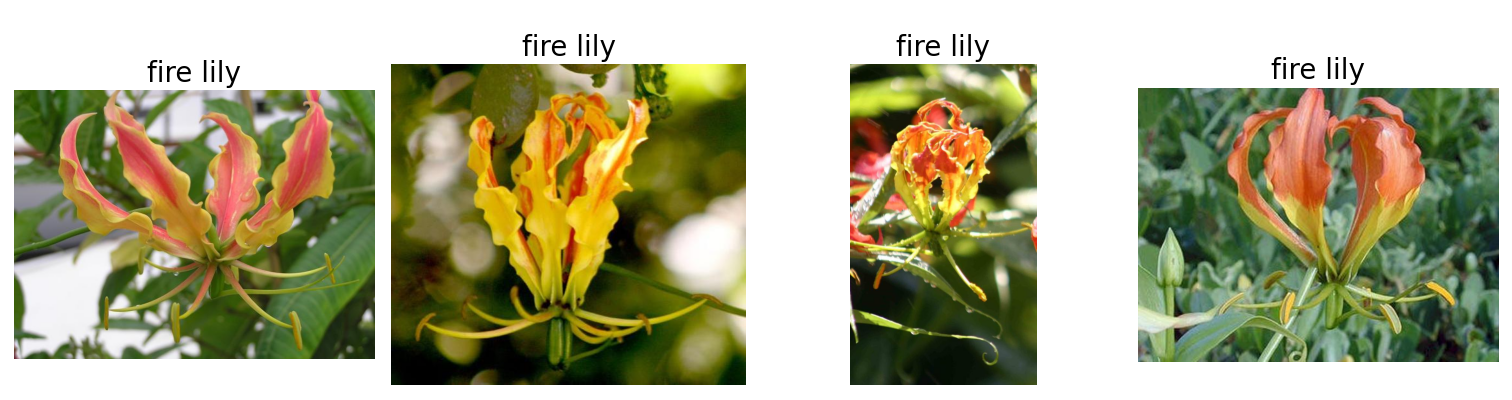}
  \caption{Four sampled images for state 20 of the \emph{FlowerGrid} environment. }
  \label{fig: flowers25_state20}
  \Description{TODO}
\end{figure*}
\begin{figure*}[tp]
  \centering
\includegraphics[width=0.6\textwidth]{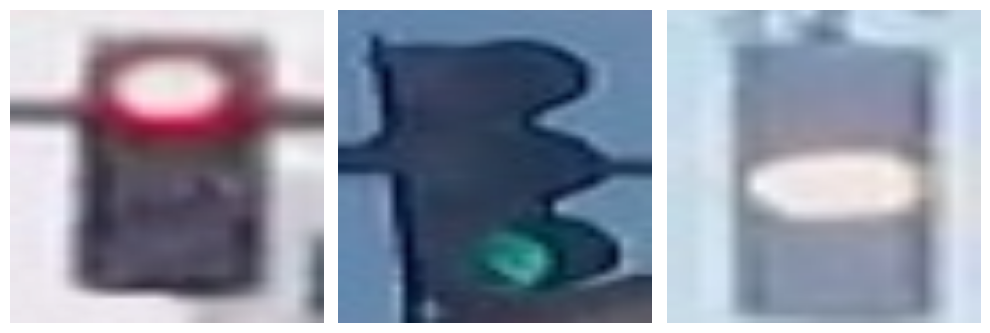}
  \caption{Images from \citet{trafficlightclassifier} used as image observations in the \emph{Intersection} environment.  }
  \label{fig: trafficlights}
  \Description{TODO}
\end{figure*}

\section{Perception models} \label{app: perception}
\paragraph{FlowerGrid}
We employed the EfficientNet-B2 architecture \citep{DBLP:conf/icml/TanL19} as the backbone of our perception model for the 102 Category Flower Dataset. 
The network was initialized with ImageNet (\cite{DBLP:conf/cvpr/DengDSLL009}) pretrained weights and modified to output predictions for 25 classes. 
Training was conducted using the Adam optimizer with a learning rate of 1e-4 and cross-entropy loss over 20 epochs.
A batch size of 5 was used, and validation accuracy was monitored after each epoch to detect overfitting.
Input images were resized to 128×128 pixels and normalized according to ImageNet standards. 
Data augmentation including random cropping, horizontal flips, rotations, and color jittering was applied during training to enhance generalization.
\paragraph{FrozenLake}
For the \emph{FrozenLake} perception task, we utilize a ResNet-18 convolutional neural network \citep{DBLP:conf/cvpr/HeZRS16}.
The network was initialized with pretrained ImageNet weights and adapted by replacing the final fully connected layer with a dropout layer (p=0.5) followed by a linear classifier mapping to 16 discrete environment states. 
The model was trained using the Adam optimizer with a learning rate of 1e-3, and categorical cross-entropy as the loss function. 
Training was performed over 20 epochs with early stopping based on validation accuracy, using a patience of two epochs.
All input images were resized to 224×224 pixels and normalized using standard ImageNet statistics. 
To improve generalization and mitigate overfitting, data augmentation was applied during training.
\paragraph{Intersection}
For the \emph{Intersection} environment, we train a ResNet-18 convolutional neural network \citep{DBLP:conf/cvpr/HeZRS16}.
The network is initialized with ImageNet-pretrained weights and adapted by replacing the final fully connected layer with a 128-dimensional hidden layer, followed by ReLU activation, dropout ($p=0.3$), and a final linear classifier over three classes.
We freeze the convolutional backbone during training and optimize only the classification head using the Adam optimizer with a learning rate of $10^{-3}$.
Training is performed over 5 epochs using cross-entropy loss, with an 80/20 train-validation split.
All images are resized to $224 \times 224$ pixels and normalized using ImageNet statistics.
\subsection{Temperature Scaling}
To calibrate the confidence estimates of the neural network predictions, we applied temperature scaling—a post-processing technique that adjusts the softmax outputs by dividing the logits by a learned scalar temperature parameter 
$T$, implementated by \citep{pleiss2017temperature}.
For the \emph{FrozenLake} dataset, the best calibration was achieved with $T=1.222$, while for the \emph{FlowerGrid} dataset, the optimal value was 
$T=0.984$. For the \emph{Intersection} the optimal temperature was $T=0.961$.

\subsection{Accuracy and Uncertainty of the perception models} \label{app:acc-and-unc-perc}
In \cref{fig:overview_accuracy}, we compare the accuracy and uncertainty of the perception models for the vision datasets under different degrees of visual corruption.
\begin{figure*}[tp]
  \centering
  \begin{subfigure}[b]{0.45\textwidth}
    \centering
    \includegraphics[width=\linewidth]{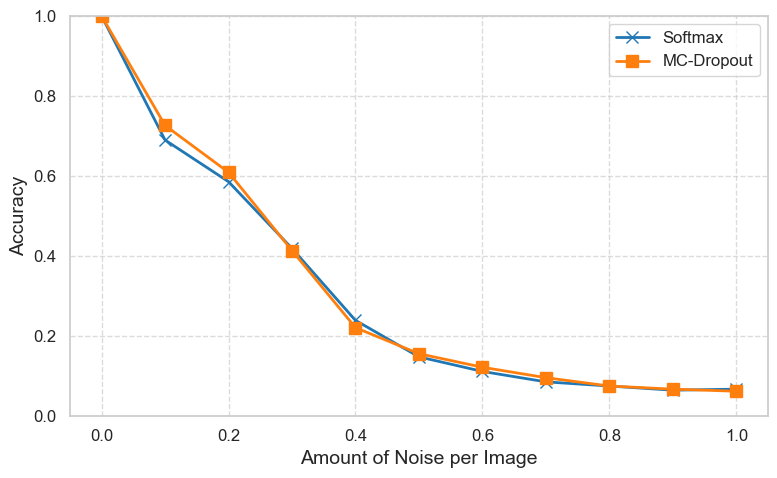}
    \caption{\emph{FrozenLake(4)} Accuracy}
    \label{fig:traf_acc}
  \end{subfigure}
  \hfill
  \begin{subfigure}[b]{0.45\textwidth}
    \centering
    \includegraphics[width=\linewidth]{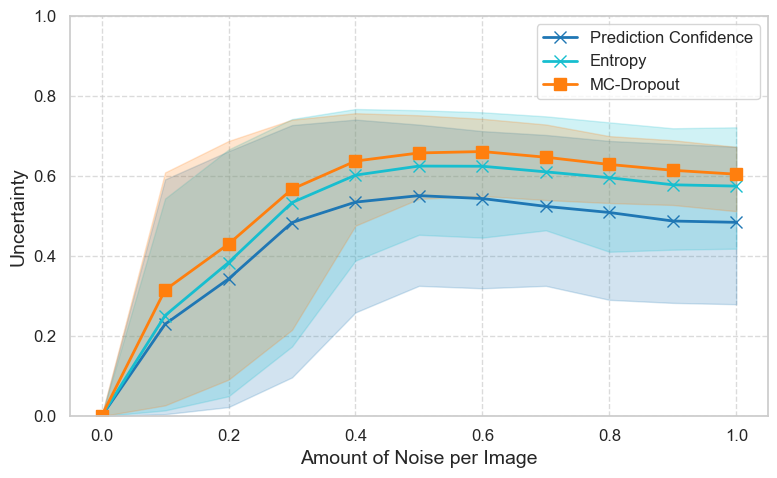}
    \caption{\emph{FrozenLake(4)} Uncertainty}
    \label{fig:traf_ent}
  \end{subfigure}

  \vspace{6pt}

  \begin{subfigure}[b]{0.45\textwidth}
    \centering
    \includegraphics[width=\linewidth]{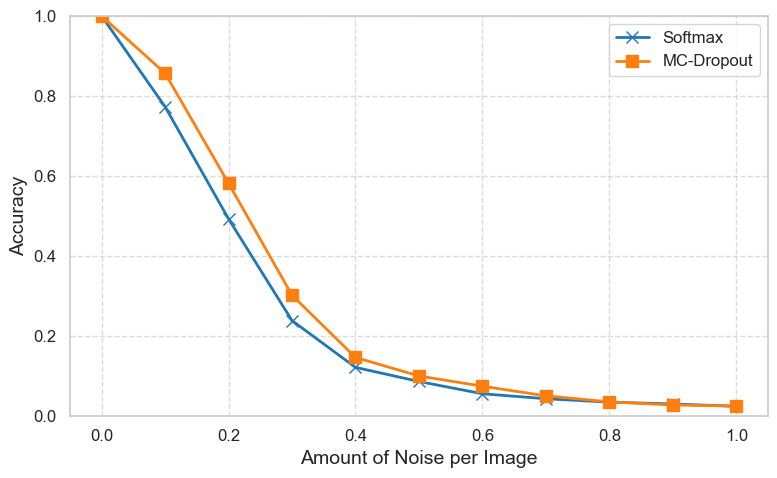}
    \caption{\emph{FrozenLake(8)} Accuracy}
    \label{fig:flow_acc}
  \end{subfigure}
  \hfill
  \begin{subfigure}[b]{0.45\textwidth}
    \centering
    \includegraphics[width=\linewidth]{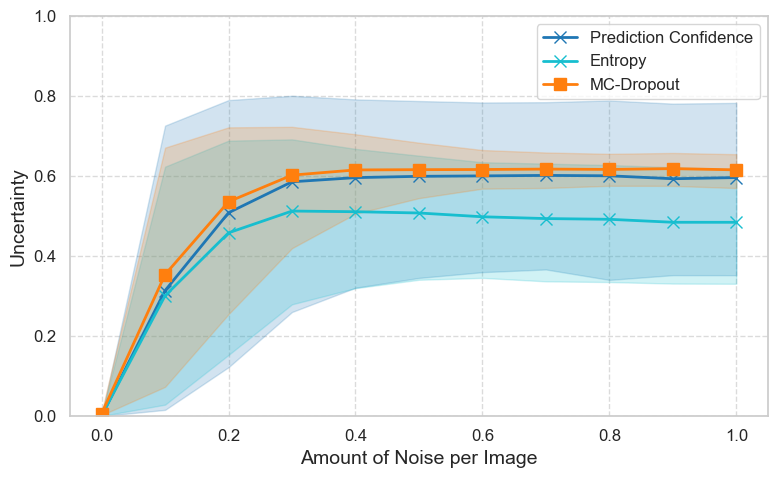}
    \caption{\emph{FrozenLake(8) Uncertainty}}
    \label{fig:flow_ent}
  \end{subfigure}

  \vspace{6pt}

  \begin{subfigure}[b]{0.45\textwidth}
    \centering
    \includegraphics[width=\linewidth]{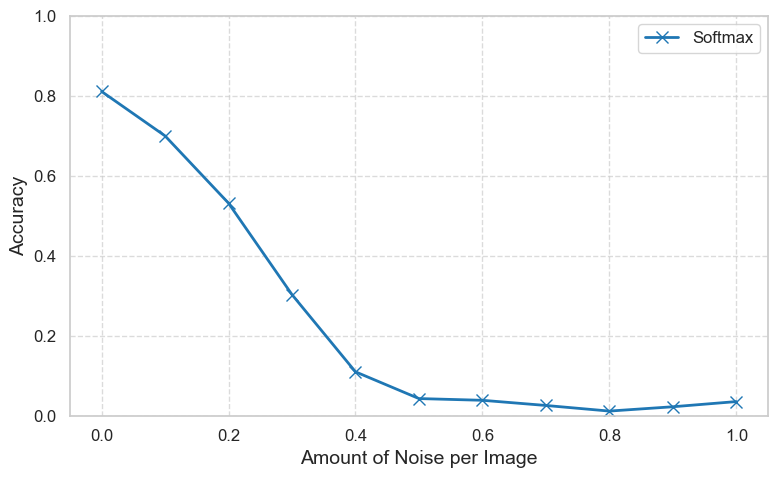}
    \caption{\emph{FlowerGrid} Accuracy}
    \label{fig:fl_acc}
  \end{subfigure}
  \hfill
  \begin{subfigure}[b]{0.45\textwidth}
    \centering
    \includegraphics[width=\linewidth]{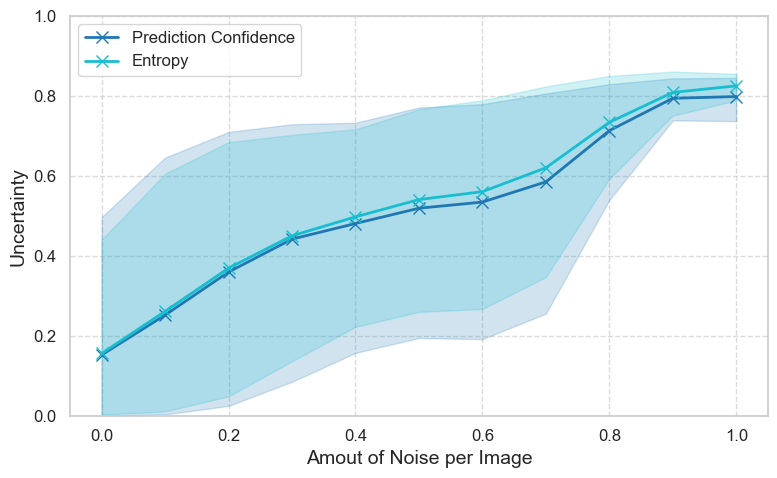}
    \caption{\emph{FlowerGrid} Uncertainty}
    \label{fig:fl_conf}
  \end{subfigure}

  \vspace{6pt}

  \begin{subfigure}[b]{0.45\textwidth}
    \centering
    \includegraphics[width=\linewidth]{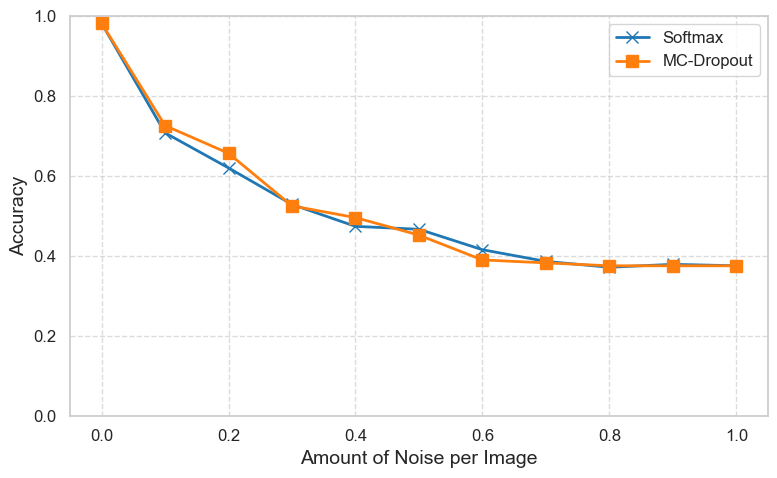}
    \caption{\emph{Intersection} Accuracy}
    \label{fig:fl_ent_sm}
  \end{subfigure}
  \hfill
  \begin{subfigure}[b]{0.45\textwidth}
    \centering
    \includegraphics[width=\linewidth]{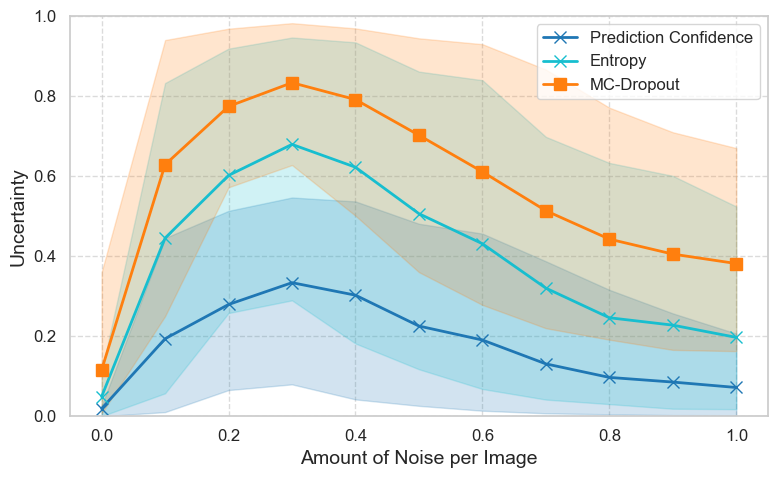}
    \caption{\emph{Intersection} Uncertainty}
    \label{fig:mcdo_ent}
  \end{subfigure}

  \caption{Comparison of the accuracy and uncertainty of the perception model on our benchmarks to the amount of pixels corrupted by salt-and-pepper noise.}
  \label{fig:overview_accuracy}
\end{figure*}

\section{Technical specifications}\label{app:exp-tech}

Our experiments were run on two machines: one more well-suited to GPU tasks, and one more well-suited for CPU tasks.
We provide the specifications of both:

\paragraph{Machine 1 ($\dqn$ algorithm):}
\begin{itemize}
    \item CPU: Intel Core i7-1260P, 12th gen
    \begin{itemize}
        \item 12 cores, 24 threads
        \item 448KiB L1d cache, 640KiB L1i cache, 9MiB L2 cache, 18MiB L3 cache.
    \end{itemize}
    \item RAM: 2x 16GiB SODIMM DDR4 3200 MHz
    \item GPU: NVIDIA Quadro T550
    \item Operating System: Ubuntu 22.04.5 LTS
\end{itemize}

\paragraph{Machine 2 (all other algorithms):}
\begin{itemize}
    \item CPU: 3.00 GHz Intel Core i9-10980XE
    \begin{itemize}
        \item 18 cores, 36 threads
        \item 576KiB L1d cache, 576KiB L1i cache, 18MiB L2 cache, 24.8MiB L3 cache.
    \end{itemize}
    \item RAM: 8x 32GiB DIMM DDR4 3200 MHz
    \item GPU: NVIDIA GeForce RTX3090
    \item Operating System: Ubuntu 22.04.1 LTS.
\end{itemize}

All code is written for Python 3.10, and our codebase includes instructions on how to install all relevant packages.

\section{Additional experimental results}\label{app:exp-more}

In this appendix, we discuss show the results of some of the experiments on $\tpomcp$ that did not make it into the main paper.

\textbf{Robustness to visual corruption}
\Cref{fig:pomcpnoise} shows the performance of $\tpomcp$, as compared to the other algorithms form our experiments, at different noise probabilities.
As expected, it's performance is significantly worse than that of all the $\pbphsvi$ variants.
However, it's performance does noet decrease as quickly as for some of the other algorithms.
Most notably, it's performance drops only slightly for $\envflowers$, and at high noise probabilities it's performance is higher than that of $\perchsvi$ and $\dqn$ in both variants of $\envlake$.

\begin{figure*}[tb]
    \centering
    \includegraphics[width=0.9\textwidth]{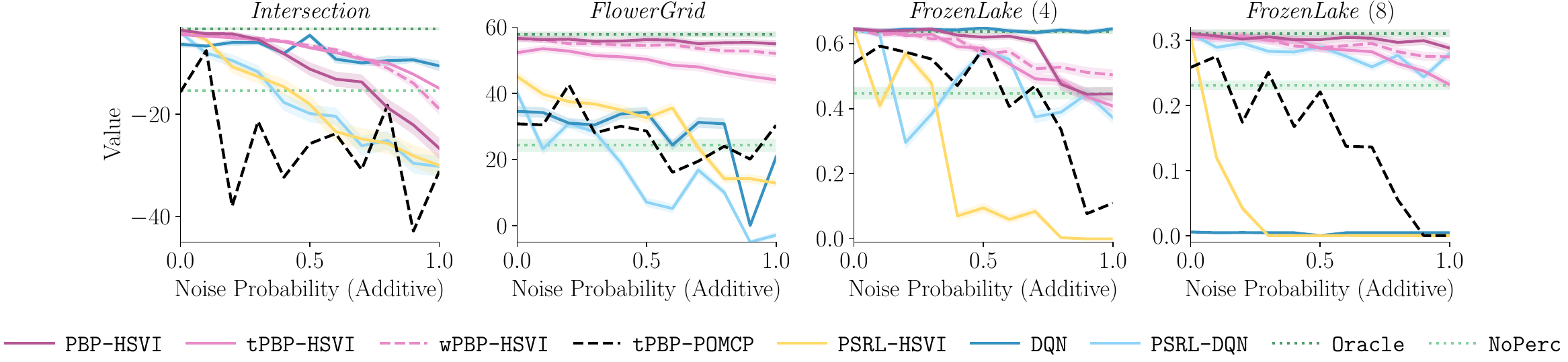}
    \caption{
    Average discounted returns (\emph{Value}) for different algorithms (including $\tpomcp$) at different probabilities of receiving noisy observations.
    Additive noise refers to images that are correctly classified with $0.4$ probability.
    Shaded areas show $95\%$ confidence in the value of the tested policy.
    }
    \label{fig:pomcpnoise}
    \Description{TODO}
\end{figure*}

\textbf{Belief distance}
As mentioned in \cref{subsec:exp_performance}, \cref{fig:beliefdist} shows the average $L^1$ distance between the real belief and the belief as expressed by $\tpomcp$'s particle filter.
Intuitively, the two beliefs are equal if the distance is $0$, and have no overlap in support if the distance is $2$.
Due to random particle sampling and particle invigoration, we never expect the distance to drop to $0$.
Indeed, as expected, the distance is slightly higher than 2 times the invigoration rate of $0.05$ for low noise levels.
Thus, at these noise probabilities, the belief update is unlikely to be responsible for the sub-par performance.
Distance increases for higher noise levels, but stays below $0.4$ for all environments but $\envlake~(8)$.
Even for $\envlake~(8)$, at the highest distance ($0.7$), the probability mass of the two beliefs still have an overlap of $65$\%.

\begin{figure*}[tb]
    \centering
    \includegraphics[width=0.7\textwidth]{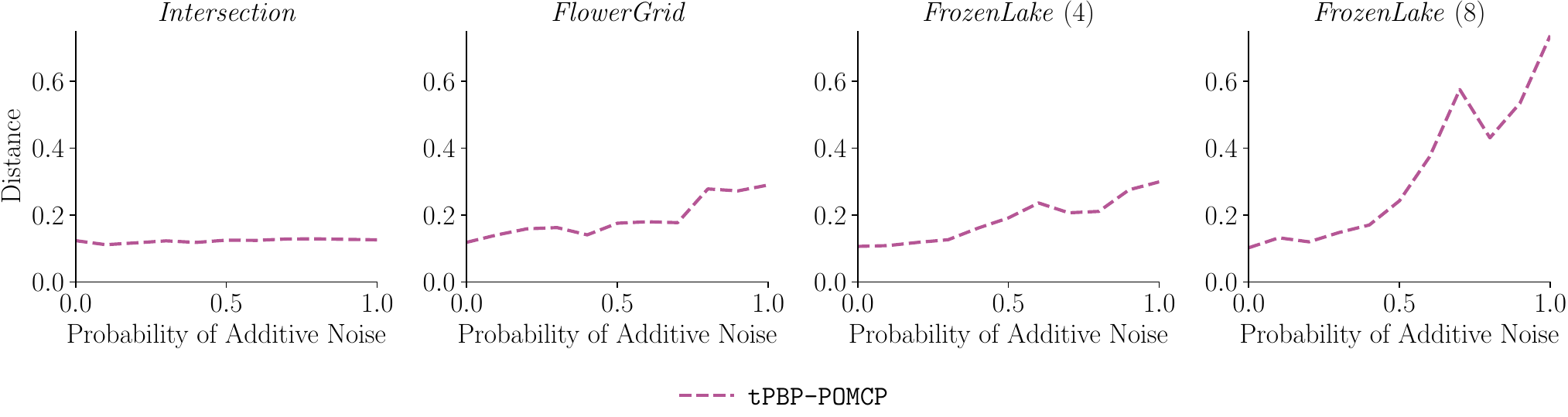}
    \caption{
    Average $L^1$ distance between belief as computed by $\thsvi$ and the belief as expressed by the particle filter of $\tpomcp$.
    }
    \label{fig:beliefdist}
    \Description{TODO}
\end{figure*}

}
{
}

\end{document}